%% file: ltexpprt.tex
\documentclass[twoside,leqno,twocolumn]{article}
\usepackage{ltexpprt}

\usepackage{amsmath}
\usepackage{amssymb}
\usepackage{verbatim}
\usepackage{graphicx} 
\usepackage{subfigure} 
\usepackage{algorithm}
\usepackage[noend]{algorithmic}
\usepackage{hyperref} 
\usepackage{todonotes}

\newcommand*\samethanks[1][\value{footnote}]{\footnotemark[#1]}

\DeclareMathOperator{\SEP}{SEP}
\DeclareMathOperator{\DISP}{DISP}
\DeclareMathOperator{\DBCVI}{DBCVI}
\DeclareMathOperator{\polylog}{polylog}

\DeclareMathOperator{\cut}{cut}
\DeclareMathOperator{\remove}{remove}
\DeclareMathOperator{\pushBack}{pushBack}
\DeclareMathOperator{\pushFront}{pushFront}
\DeclareMathOperator{\popFront}{popFront}
\DeclareMathOperator{\evaluateCut}{evaluateCut}
\DeclareMathOperator{\deque}{deque}
\DeclareMathOperator{\getOtherNodeEdge}{getOtherNodeEdge}

\begin{document}

\title{\Large Graph sketching-based Space-efficient Data Clustering}

\author{Anne Morvan\thanks{CEA, LIST, 91191 Gif-sur-Yvette, France.} \thanks{Universit\'e Paris-Dauphine, PSL Research University, CNRS, LAMSADE, 75016 Paris, France.} \thanks{Partly supported by the DGA (French Ministry of Defense).} \\ \texttt{anne.morvan@cea.fr} \\
\and
Krzysztof Choromanski\thanks{Google Brain Robotics, New York, USA.} \\ \texttt{kchoro@google.com }
\\
\and C\'edric Gouy-Pailler\samethanks[1] \\ \texttt{cedric.gouy-pailler@cea.fr}
\\
\and Jamal Atif\samethanks[2] \\ \texttt{jamal.atif@dauphine.fr} }
\date{}

\maketitle


\fancyfoot[R]{\footnotesize{\textbf{Copyright \textcopyright\ 2018 by SIAM\\
Unauthorized reproduction of this article is prohibited}}}





\begin{abstract} \small\baselineskip=9pt 
In this paper, we address the problem of recovering arbitrary-shaped data clusters from datasets while facing \emph{high space constraints}, as this is for instance the case in many real-world applications when analysis algorithms are directly deployed on resources-limited mobile devices collecting the data.
We present DBMSTClu a new space-efficient density-based \emph{non-parametric} method working on a Minimum Spanning Tree (MST) recovered from a limited number of linear measurements \emph{i.e.} a \emph{sketched} version of the dissimilarity graph $\mathcal{G}$ between the $N$ objects to cluster. Unlike $k$-means, $k$-medians or $k$-medoids algorithms, it does not fail at distinguishing clusters with particular forms thanks to the property of the MST for expressing the underlying structure of a graph. No input parameter is needed contrarily to DBSCAN or the Spectral Clustering method. An approximate MST is retrieved by following the dynamic \emph{semi-streaming} model in handling the dissimilarity graph $\mathcal{G}$ as a stream of edge weight updates which is sketched in one pass over the data into a compact structure requiring $O(N \polylog(N))$ space, far better than the theoretical memory cost $O(N^2)$ of $\mathcal{G}$. The recovered approximate MST $\mathcal{T}$ as input, DBMSTClu then successfully detects the right number of nonconvex clusters by performing relevant cuts on $\mathcal{T}$ in a time linear in $N$. We provide theoretical guarantees on the quality of the clustering partition and also demonstrate its advantage over the existing state-of-the-art on several datasets.
\end{abstract}

\section{Introduction}

Clustering is one of the principal data mining tasks consisting in grouping related objects in an unsupervised manner. It is expected that objects belonging to the same cluster are more similar to each other than to objects belonging to different clusters. There exists a variety of algorithms performing this task. Methods like $k$-means~\cite{LloydKmeans}, $k$-medians~\cite{Jain1988Kmedian} or $k$-medoids~\cite{kmedoids87} are useful unless the number and the shape of clusters are unknown which is unfortunately often the case in real-world applications. They are typically unable to find clusters with a nonconvex shape. Although DBSCAN~\cite{Ester96DBSCAN} does not have these disadvantages, its resulting clustering still depends on the chosen parameter values.

One of the successful approaches relies on a graph representation of the data. Given a set of $N$ data points $\{x_1, \dots, \ x_N \}$, a graph can be built based on the dissimilarity of data where points of the dataset are the vertices and weighted edges express distances between these objects. Besides, the dataset can be already a graph $\mathcal{G}$ modeling a network in many fields, such as bioinformatics - where gene-activation dependencies are described through a network - or social, computer, information, transportation network analysis. The clustering task consequently aims at detecting clusters as groups of nodes that are densely connected with each other and sparsely connected to vertices of other groups. In this context, Spectral Clustering (SC)~\cite{SCreview2011} is a popular tool to recover clusters with particular structures for which classical $k$-means algorithm fails. When dealing with large scale datasets, a main bottleneck of the technique is to perform the partial eigendecomposition of the associated graph Laplacian matrix, though. Another inherent difficulty is to handle the huge number of nodes and edges of the induced dissimilarity graph: storing all edges can cost up to $O(N^2)$ where $N$ is the number of nodes. Over the last decade, it has been established that the dynamic streaming model~\cite{Muthukrishnan2005} associated with linear sketching techniques~\cite{Ahn2012AGS} - also suitable for distributed processing -, is a good way for tackling this last issue. In this paper, the addressed problem  falls within a framework of storage limits allowing $O(N \polylog (N))$ space complexity but not $O(N^2)$. 

\paragraph{Contributions.} The new clustering algorithm DBMSTClu presented in this paper brings a solution to these issues: 1) detecting arbitrary-shaped data clusters, 2) with no parameter, 3) in a time linear to the number of points, 4) in a space-efficient manner by working on a limited number of linear measurements, a \emph{sketched} version of the streamed dissimilarity graph $\mathcal{G}$. 
DBMSTClu returns indeed a partition of the $N$ points to cluster relying only on a Minimum Spanning Tree (MST) of the dissimilarity graph $\mathcal{G}$ taking $O(N)$ space. 
This MST can be space-efficiently approximatively retrieved in the dynamic semi-streaming model by handling $\mathcal{G}$ as a stream of edge weight updates sketched in only one pass over the data into a compact structure taking $O(N \polylog (N))$ space. DBMSTClu then automatically identifies the right number of nonconvex clusters by cutting suitable edges of the resulting approximate MST in $O(N)$ time.

The remaining of this paper is organized as follows. In \S\ref{sec:graphsketching_related_work} the related work about graph clustering and other space-efficient clustering algorithms is described. \S\ref{sec:sketching} gives fundamentals of the sketching technique that can be used to obtain an approximate MST of the dissimilarity graph. DBMSTClu (DB for Density-Based), the proposed MST-based algorithm for clustering is then explained in \S\ref{sec:DBMSTClu_principle}, its theoretical guarantees discussed in \S\ref{sec:quality} and the implementation enabling its scalability detailed in \S\ref{sec:implementation}. \S\ref{sec:results} presents the experimental results comparing the proposed clustering algorithm to other existing methods. Finally, \S\ref{sec:ccl} concludes the work and discusses future directions.

\section{Related work} \label{sec:graphsketching_related_work}

\subsection{General graph clustering.}
The approach of graph representation of the data has led to an extensive literature over graph clustering related to graph partitioning~\cite{Schaeffer2007SGC}. 
From the clustering methods point of view, DenGraph~\cite{Falkowski2007DenGraph} proposes a graph version of DBSCAN which is able to deal with noise while 
work from~\cite{BreakingGraphClucorr} focuses on the problem of recovering clusters with considerably dissimilar sizes.  Recent works include also approaches from convex optimization using low-rank decomposition of the adjacency matrix~\cite{LowRankSparse2011, ClusteringSparseGraphsNIPS2012, Chen2014CPO, Chen2014WGC}. 
These methods bring theoretical guarantees about the exact recovery of the ground truth clustering for the Stochastic Block Model~\cite{SBM83,CondonPlanted2001, rohe2011} but demand to compute the eigendecomposition of a $N \times N$ matrix (resp. $O(N^3)$ and $O(N^2)$ for time and space complexity). Moreover they are restricted to unweighted graphs (weights in the work from~\cite{Chen2014WGC} refer to likeliness of existence of an edge, not a distance between points). 

\subsection{MST-based graph clustering.}
The MST is known to help recognizing clusters with arbitrary shapes. Clustering algorithms from this family identify clusters by performing suitable cuts among the MST edges. The first algorithm, called Standard Euclidean MST (SEMST) is from~\cite{Zahn1971} and given a number of expected clusters, consists in deleting the heaviest edges from the Euclidean MST of the considered graph but this completely fails when the intra-cluster distance is lower than the inter-clusters one. 
For decades since MST-based clustering methods~\cite{Asano1988, MSTBasedClusteringAlgosICTAI2006} have been developed and can be classified into the group of density-based methods. MSDR~\cite{MSTBasedClusteringAlgosICTAI2006} relies on the mean and the standard deviation of edge weights within clusters but will encourage clusters with points far from each other as soon as they are equally ``far". Moreover, it does not handle clusters with less than three points. 

\subsection{Space-efficient clustering algorithms.} 
Streaming $k$-means~\cite{NIPS09StreamingKmeans} is a one-pass streaming method for the $k$-means problem but still fails to detect clusters with nonconvex shapes since only the centroid point of each cluster is stored. This is not the case of CURE algorithm~\cite{CureGuha01} which represents each cluster as a random sample of data points contained in it but this offline method has a prohibitive time complexity of $O(N^2 \log(N))$ not suitable for large datasets. More time-efficient, CluStream~\cite{CluStreamAggarwal03} and DenStream~\cite{Cao06DenStream} create microclusters based on local densities in an online fashion and aggregate them later to build bigger clusters in offline steps. Though, only DenStream can capture non-spherical clusters but needs parameters like DBSCAN from which it is inspired. 

\section{Context, notations and graph sketching preprocessing step}
\label{sec:sketching}

Consider a dataset with $N$ points. Either the underlying network already exists, or it is assumed that a dissimilarity graph $\mathcal{G}$ between points can be built where points of the dataset are the vertices and weighted edges express distances between these objects. For instance, this can be the Euclidean distance. 
In both cases, the graphs considered here should follow this definition:

\begin{Definition}[graph $\mathcal{G} = (V, \ E)$] 
{ \rm
A graph $\mathcal{G} = (V, \ E)$ consists in a set of nodes $V$ and a set of edges $E \subseteq V \times V$. 
The graph is undirected but weighted. The weight $w$ on an edge between node $i$ and $j$ - if this edge exists - corresponds to the normalized predefined distance between $i$ and $j$, s.t. $0 < w \leq 1$. $|V| = N$ and $|E| = M$ stand resp. for the cardinality of sets $V$ and $E$. $E = \{e_1, \ ..., \ e_M\}$ and for all edges $e_i$ a weight $w_i$ represents a distance between two vertices. In the sequel, $E(\mathcal{G})$ describes the set of edges of a graph $\mathcal{G}$. 
}
\end{Definition}

Freely of any parameter, DBMSTClu performs the nodes clustering from its unique input: an MST of $\mathcal{G}$. So an independent preprocessing is required for its recovery by any existing method. To respect our space restrictions though, the use of a graph sketching technique is motivated here: from the stream of its edge weights, a sketch of $\mathcal{G}$ is built and then an approximate MST is retrieved exclusively from it.  

\paragraph{Streaming graph sketching and approximate MST recovery.}

Processing data in the dynamic streaming model~\cite{Muthukrishnan2005} for graph sketching implies: 1)~The graph should be handled as a stream $s$ of edge weight updates: $s = (a_1, ... \ a_j, ...)$ where $a_j$ is the $j$-th update in the stream corresponding to the tuple $a_j = (i, \ w_{old, i}, \ \Delta w_i)$ with $i$ denoting the index of the edge to update, $w_{old, i}$ its previous weight and $\Delta w_i$ the update to perform. Thus, after reading $a_j$ in the stream, the $i$-th edge is assigned the new weight $w_i = w_{old, i} + \Delta w_i \geq 0$. 2)~The method should make only one pass over this stream . 3)~Edges can be both inserted or deleted (\emph{turnstile} model), i.e. weights can be increased or decreased (but have always to be nonnegative). So weights change regularly, as in social networks where individuals can be friends for some time then not anymore. 

The algorithm in~\cite{Ahn2012AGS} satisfies these conditions and is used here to produce in an online fashion a limited number of linear measurements summarizing edge weights of $\mathcal{G}$, as new data $a_j$ are read from the stream $s$ of edge weight updates. Its general principle is briefly described here. For a given small $\epsilon_1$, $\mathcal{G}$ is seen as a set of unweighted subgraphs $\mathcal{G}_k$ containing all the edges with weight lower than $(1 + \epsilon_1)^k$, hence $\mathcal{G}_k \subset \mathcal{G}_{k+1}$. The $\mathcal{G}_k$ are embodied as $N$ virtual vectors $v^{(i)} \in \{-1, 0, 1 \}^M$ for $i \in [N]$\footnote{In the sequel, for a given integer $a$, $[a] = \{1, \ldots, a \}$.} expressing for each node the belonging to an existing edge: for $j \in [M]$, $v^{(i)}_j$ equals to $0$ if node $i$ is not in $e_j$, $1$ (resp. $-1$) if $e_j$ exists and $i$ is its left (resp. right) node. All $v^{(i)}$ are described at $L$ different ``levels", i.e. $L$ virtual copies of the true vectors are made with some entries randomly set to zero s.t. the $v^{(i), l}$ get sparser as the corresponding level $l\in [L]$ increases. The $v^{(i), l}$ for each level are explicitly coded in memory by three counters: 
$\phi = \sum_{j = 1}^M v^{(i), l}_j$;
$\iota = \sum_{j = 1}^M j \ v^{(i),l}_j$;
$\tau = \sum_{j= 1}^M v^{(i),l}_j \ z^j \mod p$, with $p$ a suitably large prime and $z \in \mathbb{Z}_{p}$.
The resulting compact data structure further named a \emph{sketch} enables to draw almost uniformly at random a nonzero weighted edge among $\mathcal{G}_k$ at any time among the levels vectors $v^{(i), l}$ which are $1$-sparse (with exactly one nonzero coefficient) thanks to $\ell_0$-sampling~\cite{CormodeFirmani13J}:

\begin{Definition}[$\ell_0$-sampling]
{ \rm
An $(\epsilon, \delta)$ $\ell_0$-sampler for a nonzero vector $x \in \mathbb{R}^n$ fails with a probability at most $\delta$ or returns some $i \in [n]$ with probability 
$(1 \pm \epsilon) \frac{1}{| \operatorname{supp} x |}$
where $\operatorname{supp} x = \{ i \in [n] \ | \ x_i \neq 0 \} $.
}
\end{Definition}

The sketch requires $O(N \log^3(N))$ space.
It follows that the sketching is technically \emph{semi-streamed} but in practice only one pass over the data is needed and the space cost is significantly lower than the theoretical $O(N^2)$ bound. The time cost for each update of the sketch is $\polylog(N)$. 
The authors from~\cite{Ahn2012AGS} also proposed an algorithm to compute in a single-pass the approximate weight $\tilde{W}$ of an MST $\mathcal{T}$ - the sum of all its edge weights - by appropriate samplings from the sketch in $O(N \polylog(N))$ time. 
They show that $W \leq \tilde{W} \leq (1 + \epsilon_1) \ W$ where $W$ stands for the true weight and $\tilde{W} = N - (1+\epsilon_1)^{r+1} \ cc(\mathcal{G}_r) + \sum^r_{k = 0} \lambda_k \ cc(\mathcal{G}_k) $
with $\lambda_k = (1 + \epsilon_1)^{k+1} - (1 + \epsilon_1)^i$, $r = \lceil \log_{1+\epsilon_1}(w_{max}) \rceil$ s.t. $w_{max}$ is the maximal weight of $\mathcal{G}$ and $cc$ denotes the number of connected components of the graph in parameter.
Here an extended method is applied for obtaining rather an approximate MST - and not simply its weight - by registering edges as they are sampled. Referring to the proof of Lemma 3.4 in~\cite{Ahn2012AGS}, the approach is simply justified by applying Kruskal's algorithm where edges with lower weights are first sampled\footnote{We would like to thank Mario Lucic for the fruitful private conversation and for coauthoring with Krzysztof Choromanski the MSE sketching extension during his internship at Google.}. 

Note that the term MST is kept in the whole paper for the sake of simplicity, but the sketching technique and so does our algorithm enables to recover a Minimum Spanning Forest if the initial graph is disconnected. 

\section{The proposed MST-based graph clustering method: DBMSTClu}
\label{sec:DBMSTClu}

\subsection{Principle.}
\label{sec:DBMSTClu_principle}

Let us consider a dataset with $N$ points. After the sketching phase, an approximate MST further named $\mathcal{T}$ has been obtained with $N-1$ edges s.t. $\forall i \in [N-1], \ 0 < w_i \leq 1$. Our density-based clustering method DBMSTClu exclusively relies on this object by performing some cuts among the edges of the tree s.t. $K-1$ cuts result in $K$ clusters. Note that independently of the technique used to obtain an MST (the sketching method is just a possible one), the space complexity of the algorithm is $O(N)$ which is better than the $O(N^2)$ of Spectral Clustering (SC). The time complexity of DBMSTClu is $O(NK)$ which is clearly less than the $O(N^3)$ one implied by SC.
After a cut, obtained clusters can be seen as subtrees of the initial $\mathcal{T}$ and the analysis of their qualities is only based on edges contained in those subtrees. In the sequel, all clusters $C_i$, $i \in [K]$, are assimilated to their associated subtree of $\mathcal{T}$ and for instance the maximal edge of a cluster will refer to the edge with the maximum weight from the subtree associated to this cluster.

Our algorithm is a parameter-free divisive top-down procedure: 
it starts from one cluster containing the whole dataset and at each iteration, a cut which maximizes some criterion is performed. 
The criterion for identifying the best cut to do (if any should be made) at a given stage is a measure of the \emph{validity} of the resulting clustering partition. This is a function of two positive quantities defined below: \emph{Dispersion} and \emph{Separation} of one cluster. The quality of a given cluster is then measured from Dispersion and Separation while the quality of the clustering partition results from the weighted average of all cluster validity indices. Finally, all those measures are based on the value of edge weights and the two latter ones lie between $-1$ and $1$. 

\begin{Definition}[Cluster Dispersion]
{ \rm
The Dispersion of a cluster $C_i$ ($\DISP$) is defined as the maximum edge weight of $C_i$. If the cluster is a singleton (i.e. contains only one node), the associated Dispersion is set to $0$. More formally:
\begin{equation}
\forall i \in [K], \ \DISP(C_i) = \left\{
    \begin{array}{ll}
        \max\limits_{j, \ e_j \in C_i} w_j & \mbox{if } \ |E(C_i)| \neq 0 \\
        0 & \mbox{otherwise.}
    \end{array}
\right.
\end{equation}
}
\end{Definition}

\begin{Definition}[Cluster Separation] 
{\rm
The Separation of a cluster $C_i$ ($\SEP$) is defined as the minimum distance between the nodes of $C_i$ and the ones of all other clusters $C_j, j \neq i, 1 \leq i,j \leq K, K \neq 1$ where $K$ is the total number of clusters. In practice, it corresponds to the minimum weight among all already cut edges from $\mathcal{T}$ comprising a node from $C_i$. If $K = 1$, the Separation is set to $1$. More formally, 
with $Cuts(C_i)$ denoting cut edges incident to $C_i$,
\begin{equation}
\forall i \in [K], \ \SEP(C_i) = \left\{
    \begin{array}{ll}
        \min\limits_{j, \ e_j \in Cuts(C_i)} w_j & \mbox{if } \ K \neq 1 \\
         1 & \mbox{otherwise.}
    \end{array}
\right.
\end{equation}
}
\end{Definition}

\begin{figure}[!htbp]
\centering
\vspace*{-0.35in}
\includegraphics[width=0.3\textwidth]{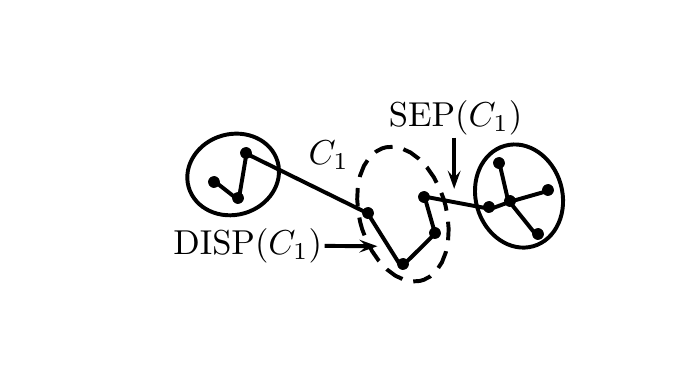}
\vspace*{-0.3in}
\caption{$\SEP$ and $\DISP$ definitions with $N = 12$, $K = 3$ for dashed cluster $C_1$ in the middle.}
\label{fig:def}
\end{figure}

Fig.~\ref{fig:def} sums up the introduced definitions.
The higher the Separation, the farther is the cluster separated from the other clusters, while low values suggest that the cluster is close to the nearest one.

\begin{Definition}[Validity Index of a Cluster]
{ \rm
The Validity Index of a cluster $C_i$ is defined as:
\begin{equation}
V_C(C_i) = \frac{ \SEP(C_i) - \DISP(C_i)}{ \max(\SEP(C_i), \DISP(C_i))} 
\end{equation}
}
\end{Definition}
The Validity Index of a Cluster (illustration in Fig.~\ref{fig:validity_index_cluster}) is defined s.t. $ -1 \leq V_C(C_i) \leq 1$ where $1$ stands for the best validity index and $-1$ for the worst one. 
No division by zero (i.e. $\max(\DISP(C_i),\SEP(C_i)) = 0$) happens because Separation is always strictly positive.
When Dispersion is higher than Separation, $-1 < V_C(C_i) < 0$. Conversely, when Separation is higher than Dispersion, $0 < V_C(C_i) < 1$. So our clustering algorithm will naturally encourage clusters with a higher Separation over those with a higher Dispersion.

\begin{figure}[!htbp]
\vspace*{-0.2in}
\centering
\includegraphics[width=0.3\textwidth]{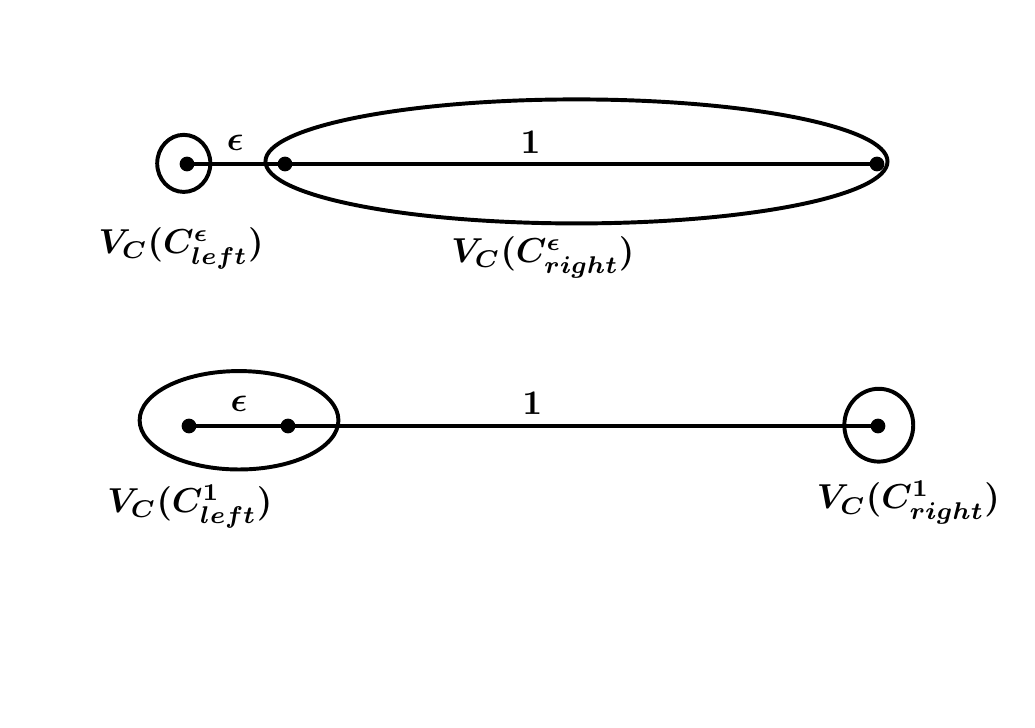}
\vspace*{-0.4in}
\caption{Validity Index of a Cluster's example with $N = 3$. For a small $\epsilon$, cutting edge with weight $\epsilon$ or $1$ gives resp. left and right partitions. (up) $V_C(C^{\epsilon}_{left}) = 1$; $V_C(C^{\epsilon}_{right}) = \epsilon - 1 < 0$. (bottom) $V_C(C^{1}_{left}) = 1 - \epsilon > 0 $; $V_C(C^{1}_{right}) = 1$. The bottom partition, for which validity indices of each cluster are positive, is preferred.}
\label{fig:validity_index_cluster}
\vspace*{-0.2in}
\end{figure}

\begin{Definition} \textsc{(Validity Index of a Clustering Partition)}
{\rm
The Density-Based Validity Index of a Clustering partition $\Pi = \{C_i\}, 1 \leq i \leq K$, $\DBCVI(\Pi)$ is defined as the weighted average of the Validity Indices of all clusters in the partition where $N$ is the number of points in the dataset.
\vspace*{-0.1in}
\begin{equation}
\DBCVI(\Pi) = \sum_{i = 1}^{K} \frac{ | C_i| }{N} V_C(C_i)
\end{equation}
}
\end{Definition}

The Validity Index of Clustering lies also between $-1$ and $1$ where $1$ stands for an optimal density-based clustering partition while $-1$ stands for the worst one. 

Our defined quantities are significantly distinct from the \emph{separation} and \emph{sparseness} defined in~\cite{MoulaviJCZS14}. Indeed, firstly, their quantities are not well defined for special cases when clusters have less than four nodes or a partition containing a lonely cluster. Secondly, the way they differentiate internal and external nodes or edges does not properly recover easy clusters like convex blobs.
Moreover, our DBCVI differs from the Silhouette Coefficient~\cite{ROUSSEEUW1987Sil}. It does not perform well with nonconvex-shaped clusters and although this is based on close concepts like \emph{tightness} and also \emph{separation}, the global coefficient is based on the average values of Silhouette coefficients of each point, while our computation of DBCVI begins at the cluster level. 

DBMSTClu is summarized in Algorithm~\ref{alg:DBMSTClu}. It starts from a partition with one cluster containing the whole dataset whereas the associated initial DBCVI is set to the worst possible value: $-1$. As long as there exists a cut which makes the DBCVI greater from (or equal to) the one of the current partition, a cut is greedily chosen by maximizing the obtained DBCVI among all the possible cuts. 
When no direct improvement is possible, the algorithm stops. 
It is guaranteed that the cut edge locally maximizes the DBCVI at each iteration since by construction, the algorithm will try each possible cut. In practice, the algorithm stops after a reasonable number of cuts, getting trapped in a local maximum corresponding to a meaningful cluster partition. This prevents from obtaining a partition where all points are in singleton clusters. Indeed, such a result ($K = N$) is not desirable, although it is optimal in the sense of the DBCVI, since in this case, $\forall i \in [K]$, $\DISP(C_i) = 0$ and $V_C(C_i)=1$.
Moreover, the non-parametric characteristic helps achieving stable partitions.
In Algorithm~\ref{alg:DBMSTClu}, $\evaluateCut(.)$ computes the DBCVI when the cut in parameter is applied to $\mathcal{T}$. 

\begin{algorithm}[!htbp]
\caption{Clustering algorithm DBMSTClu}\label{alg:DBMSTClu}
\begin{algorithmic}[1]
\STATE {\bfseries Input:} $\mathcal{T}$, the MST
\STATE $dbcvi \small{\leftarrow} -1.0$; $clusters = [ \ ]$; $cut\_list \small{\leftarrow} [E(\mathcal{T})]$
\WHILE{ $dbcvi < 1.0$ }
\STATE $cut\_tp \small{\leftarrow} None$; $dbcvi\_tp \small{\leftarrow} dbcvi$
\FOR{each $cut$ in $cut\_list$}
\STATE $newDbcvi \small{\leftarrow} \evaluateCut(\mathcal{T}, cut)$
\IF{$newDbcvi \geq dbcvi\_tp$} 
\STATE $cut\_tp \small{\leftarrow} cut$; $dbcvi\_tp \small{\leftarrow}  newDbcvi$
\ENDIF
\ENDFOR
\IF{$cut\_tp \neq None$}
\STATE $clusters \small{\leftarrow} \cut(clusters, cut\_tp)$
\STATE $dbcvi \small{\leftarrow} dbcvi\_tp$; $\remove(cut\_list, cut\_tp) $
\ELSE 
\STATE {\bfseries break}
\ENDIF
\ENDWHILE
\RETURN $clusters$, $dbcvi$
\end{algorithmic}
\end{algorithm}

\vspace*{-0.2in}
\subsection{Quality of clusters.}
\label{sec:quality}
An analysis of the algorithm and the quality of the recovered clusters is now given. The main results are: 1) DBMSTClu differs significantly from the naive approach of SEMST by preferring cuts which do not necessarily correspond to the heaviest edge (Prop.~\ref{prop:first_cut_heaviest} and \ref{prop:first_cut_middle}). 2) As long as the current partition contains at least one cluster with a negative validity index, DBMSTClu will find a cut improving the global index (Prop.~\ref{prop:negativeClusters}). 3) Conditions are given to determine in advance if and which cut will be performed in a cluster with a positive validity index (Prop.~\ref{prop:pos_index1} and \ref{prop:pos_index2} ). All are completely independent of the sketching phase.
Prop.~\ref{prop:first_cut_heaviest} and \ref{prop:first_cut_middle} rely on the two basic lemmas regarding the first cut in the MST:

\begin{lemma}\textsc{Highest weighted edge case} \label{lem:heaviest_edge} Let $\mathcal{T}$ be an MST of the dissimilarity data graph. If the first cut from $E(\mathcal{T})$ made by DBMSTClu is the heaviest edge, then resulting DBCVI is nonnegative.
\end{lemma}

\begin{proof}
For the first cut, both separations of obtained clusters $C_1$ and $C_2$ are equal to the weight of the considered edge for cut. Here, this is the one with the highest weight. Thus, for $i = 1,2$, $\DISP(C_i) \leq \SEP(C_i) \small{\implies} V_C(C_i) \geq 0$. Finally, the DBCVI of the partition, as a convex sum of two nonnegative quantities, is clearly nonnegative.
\end{proof}

\begin{lemma}\textsc{Lowest weighted edge case} \label{lem:lowest_edge} Let $\mathcal{T}$ be an MST of the dissimilarity data graph. If the first cut from $E(\mathcal{T})$ done by DBMSTClu is the one with the lowest weight, then resulting DBCVI is negative or zero. 
\end{lemma}

\begin{proof}
Same reasoning in the opposite case s.t. $\SEP(C_i) - \DISP(C_i) \leq 0$ for $i \in \{1,2\}$. 
\end{proof}

\begin{proposition}\textsc{When the first cut is not the heaviest} \label{prop:first_cut_heaviest} Let $\mathcal{T}$ be an MST of the dissimilarity data graph with $N$ nodes. Let us consider this specific case: all edges have a weight equal to $w$ except two edges $e_1$ and $e_2$ resp. with weight $w_1$ and $w_2$ s.t. $w_1 > w_2 > w > 0$. 
DBMSTClu does not cut any edge with weight $w$ and cuts $e_2$ instead of $e_1$ as a first cut iff:
\begin{equation*}
w_2 > \frac{2n_2w_1 - n_1 + \sqrt{ n_1^2 + 4 w_1( n_2^2 w_1 + N^2 - N n_1 -n_2 ^2)} }{2(N - n_1 + n_2)}
\end{equation*}
where $n_1$ (resp. $n_2$) is the number of nodes in the first cluster resulting from the cut of $e_1$ (resp. $e_2$). Otherwise, $e_1$ gets cut.
\end{proposition}
\begin{proof} 
See supplementary material.
\end{proof}

This proposition emphasizes that the algorithm is cleverer than simply cutting the heaviest edge first. Indeed, although $w_2 < w_1$, cutting $e_2$ could be preferred over $e_1$. Moreover, no edge with weight $w$ can get cut at the first iteration as they have the minimal weight in the tree.
Indeed it really happens since an approximate MST with discrete rounded weights is used when sketching is applied.  

\begin{proposition} \textsc{First cut on the heaviest edge in the middle} \label{prop:first_cut_middle} Let $\mathcal{T}$ be an MST of the dissimilarity data graph with $N$ nodes. Let us consider this specific case: all edges have a weight equal to $w$ except two edges $e_1$ and $e_2$ resp. with weight $w_1$ and $w_2$ s.t. $w_1 > w_2 > w > 0$. Denote $n_1$ (resp. $n_2$) the number of nodes in the first cluster resulting from the cut of $e_1$ (resp. $e_2$). 
In the particular case where edge $e_1$ with maximal weight $w_1$ stands between two subtrees with the same number of points, i.e. $n_1 = N/2$, $e_1$ is always preferred over $e_2$ as the first optimal cut.
\end{proposition}
\begin{proof}
See supplementary material.
\end{proof}

\begin{Remark}
\label{rem:counterExample}
Let us consider the MST in Fig.~\ref{fig:counterExample} with $N = 8$, $w_1 = 1$, $w_2 = w_3 = 1 - \epsilon$, and other weights set to $\epsilon$. Clearly, it is not always preferred to cut $e_1$ in the middle since for $\epsilon = 0.1$, $DBCVI_2 \approx 0.27 > DBCVI_1 = \epsilon = 0.1$. So, it is a counter-example to a possible generalization of Prop.~\ref{prop:first_cut_middle} where there would be more than three possible distinct weights in $\mathcal{T}$.
\begin{figure}[!htbp]
\vspace*{-0.2in}
\includegraphics[width=0.45\textwidth]{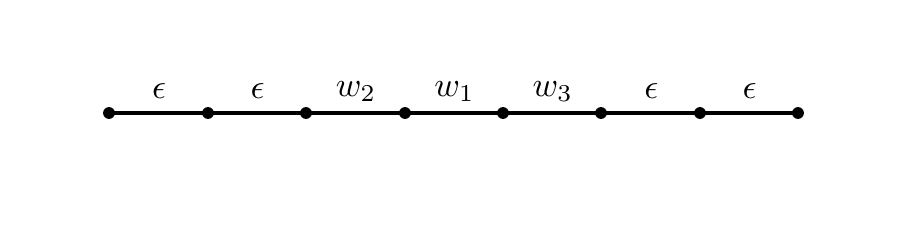}
\vspace*{-0.45in}
\caption{Counter-example for Remark \ref{rem:counterExample}}
\label{fig:counterExample}
\end{figure}
\end{Remark}
These last propositions hold for every iteration in the algorithm.

\begin{proposition}\textsc{Fate of negative $V_C$ cluster}
\label{prop:negativeClusters}
Let $K = t + 1$ be the number of clusters in the clustering partition at iteration $t$. If for some $i \in [K], \ V_C(C_i) < 0$, then DBMSTClu will cut an edge at this stage.
\end{proposition}
\begin{proof}
See supplementary material.
\end{proof}

\begin{figure}[!htbp]
\vspace*{-0.5in}
\centering
\includegraphics[width=0.34\textwidth]{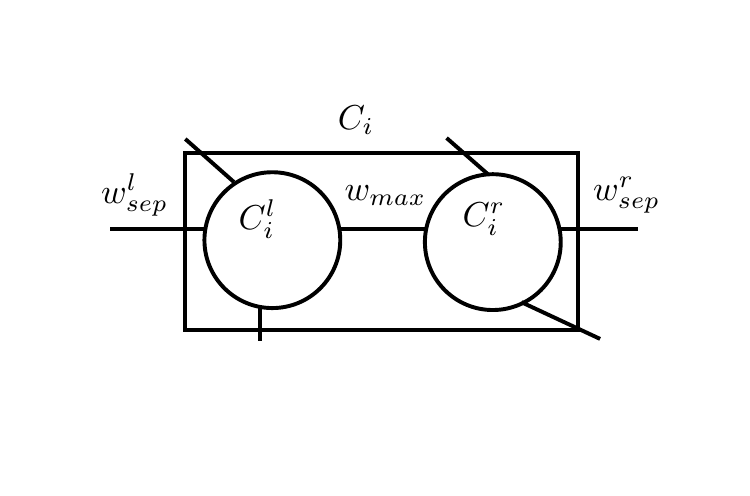}
\vspace*{-0.6in}
\caption{Generic example of Prop.~\ref{prop:negativeClusters} and \ref{prop:pos_index2}'s proofs.}
\label{fig:negativeClusters}
\end{figure}

Thus, at the end of the clustering algorithm, each cluster will have a nonnegative $V_C$ so at each step of the algorithm, this bound holds for the final DBCVI:
$DBCVI \geq \sum_{i = 1}^K \frac{ | C_i| }{N} \max(V_C(C_i), 0)$.

\begin{proposition}\textsc{Fate of positive $V_C$ cluster I} \label{prop:pos_index1} Let $\mathcal{T}$ be an MST of the dissimilarity data graph and $C$ a cluster s.t. $V_C(C) > 0$ and $\SEP(C) = s$. DBMSTClu does not cut an edge $e$ of $C$ with weight $w < s$ if both resulting clusters have at least one edge with weight greater than $w$.
\end{proposition}

\begin{proof} 
See supplementary material.
\end{proof}

\begin{proposition}\textsc{Fate of positive $V_C$ cluster II}
\label{prop:pos_index2}
Consider a partition with $K$ clusters s.t. some cluster $C_i$, $i \in [K]$ with $V_C(C_i) > 0$ is in the setting of Fig.~\ref{fig:negativeClusters} i.e. cutting the heaviest edge $e$ with weight $w_{max}$ results in two clusters: the left (resp. right) cluster $C_i^l$ (resp. $C_i^r$) with $n_1$ points (resp. $n_2$) s.t. $\DISP(C_i^l) = d_1$, $\SEP(C_i^l) = w^l_{sep}$, $\DISP(C_i^r) = d_2$ and $\SEP(C_i^r) = w^r_{sep}$. Assuming w.l.o.g. $w^l_{sep} > w^r_{sep}$, 
cutting edge $e$ improves the DBCVI iff:
\begin{equation*}
\frac{ \left(  \frac{n_1 d_1 + n_2 d_2}{n_1 + n_2} \right) }{w_{max}} \leq \frac{w_{max}}{w^r_{sep}}.
\end{equation*}
\end{proposition}

\begin{proof}
See supplementary material.
\end{proof}

\vspace*{-0.5in}
\subsection{Implementation for linear time and space complexities.}
\label{sec:implementation}

\begin{figure}[!htbp]
\vspace*{-0.3in}
\centering
\includegraphics[width=0.30\textwidth]{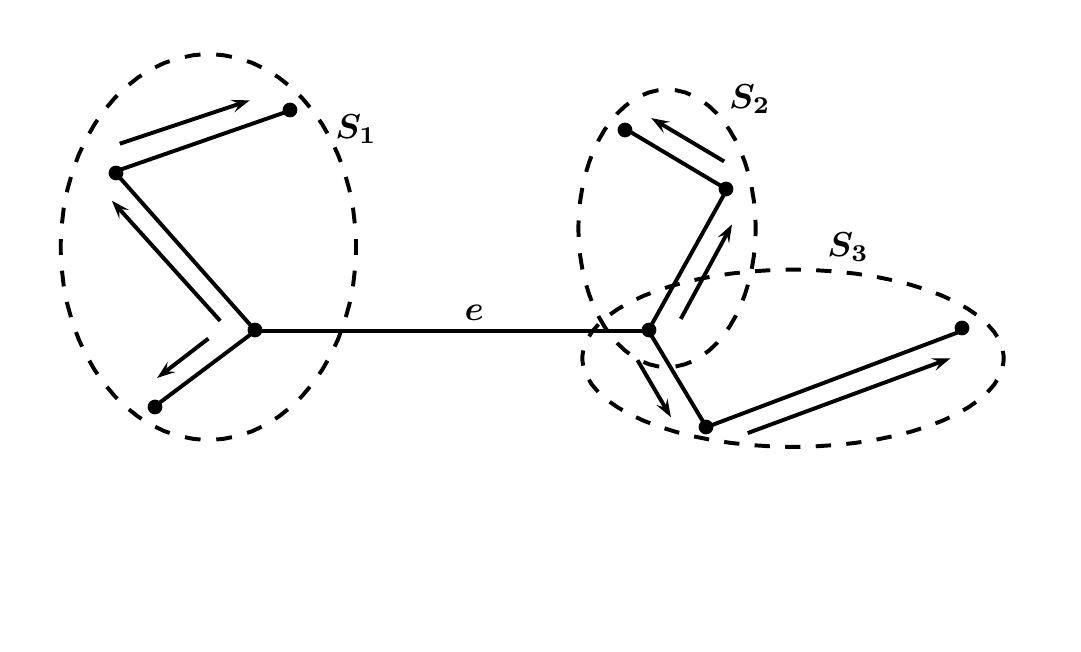}
\vspace*{-0.5in}
\caption{Illustration of the recursive relationship for left and right Dispersions resulting from the cut of edge $e$: $\DISP_{left}(e) = \max( w(S_1))$, $\DISP_{right}(e) = \max( w(S_2), w(S_3))$ where $w(.)$ returns the edge weights. Separation works analogically.}
\label{fig:recursion}
\vspace*{-0.22in}
\end{figure}

Algorithm~\ref{alg:DBMSTClu} previously described could lead to a naive implementation. We now briefly explain how to make the implementation efficient in order to achieve the linear time and space complexities in $N$
(code on \url{https://github.com/annemorvan/DBMSTClu/}).
The principle is based on two tricks.
1) As observed in~\S\ref{sec:quality}: for a performed cut in cluster $C_i$, $V_C(C_j)$ for any $j \neq i$ remain unchanged. Hence, if $V_C(C_j)$ are stored for each untouched cluster after a given cut, only the edges of $C_l$ and $C_r$ resp. the left and right clusters induced by $e$'s cut need to be evaluated again to determine the DBCVI in case of cut.
Thus the number of operations to find the optimal cut decreases drastically over time as the clusters become smaller through the cuts.
2) However finding the first cut already costs $O(N)$ time hence paying this price for each cut evaluation would lead to $O(N^2)$ operations. Fortunately, this can be avoided as $\SEP$ and $\DISP$ exhibit some recurrence relationship in $\mathcal{T}$: when knowing these values for a given cut, we can deduce the value for a neighboring cut (cf. Fig.~\ref{fig:recursion}). To determine the first cut, $\mathcal{T}$ should be hence completely crossed following the iterative version of the Depth-First Search. The difficulty though is that the recursive relationship between the quantities to update is directional: left and right w.r.t. the edge to cut. So we develop here \emph{double Depth-First search} (see principle in Algorithm~\ref{alg:doubleDFS}): from any given edge of $\mathcal{T}$, edges left and right are all visited consecutively with a Depth-First search, and $\SEP$ and $\DISP$ are updated recursively thanks to a carefully defined order in the priority queue of edges to handle.

\begin{algorithm}[!htbp]
\caption{Generic Double Depth-First Search}\label{alg:doubleDFS}
\begin{algorithmic}[1]
\STATE {\bfseries Input:} $\mathcal{T}$, the MST; $e$, the edge of $\mathcal{T}$ where the search starts; $n\_src$, source node of $e$
\STATE $Q=\deque()$ \textit{//Empty priority double-ended queue}
\FOR{incident edges to $n\_src$} 
\STATE $\pushBack(Q, (incident\_e, n\_src, FALSE))$
\ENDFOR
\STATE $\pushBack(Q, (e, n\_src, TRUE))$
\FOR{incident edges to $n\_trgt$}
\STATE $\pushBack(Q, (incident\_e, n\_trgt,FALSE))$
\ENDFOR
\STATE $\pushBack(Q, (e, n\_trgt, TRUE))$
\WHILE{$Q$ is not empty}
\STATE $e, node, marked = \popFront(Q)$
\STATE $opposite\_node = \getOtherNodeEdge(e, node)$
\IF{not $marked$}
\FOR{incident edges to $node$} 
\STATE $\pushBack(Q, (incident\_e, node,F))$
\ENDFOR
\STATE $\pushBack(Q, (e, node, T))$
\STATE $\pushFront(Q, (e, opposite\_node, T))$
\ELSE
\STATE doTheJob(e) \textit{//Perform here the task }
\ENDIF
\ENDWHILE
\RETURN 
\end{algorithmic}
\end{algorithm}

\vspace*{-0.1in}
\section{Experiments} \label{sec:results}

Tight lower and upper bounds are given in \S\ref{sec:sketching} for the weight of the approximate MST retrieved by sketching. First experiments in \S\ref{sec:safety} show that the clustering results do not suffer from the use of an approximate MST instead of a real one. Experiments from \S\ref{sec:scalability} prove then the scalability of DBMSTClu for large values of $N$. 

\subsection{Safety of the sketching.}
\label{sec:safety}

The results of DBMSTClu are first compared with DBSCAN~\cite{Ester96DBSCAN} because it can compete with DBMSTClu as 1) nonconvex-shaped clusters are recognized, 2) it does not require explicitly the number of expected clusters, 3) it is both time and space-efficient (resp. $O(N \log N)$ and $O(N)$). Then, the results of another MST-based algorithm are shown for comparison. The latter called Standard Euclidean MST (SEMST)~\cite{Zahn1971} cuts the $K-1$ heaviest edges of a standard Euclidean MST given a targeted number of clusters $K$.  
For DBMSTClu, the dissimilarity graph is built from computing the Euclidean distance between data points and passed into the sketch phase to produce an approximate version of the exact MST.

\textbf{Synthetic datasets.}
Experiments were performed on two classic synthetic datasets from the Euclidean space: noisy circles and noisy moons. Each dataset contains $1000$ data points in $20$ dimensions: the first two dimensions are randomly drawn from predefined 2D-clusters, as shown in Fig.~\ref{fig:noisycircles_approx} and \ref{fig:noisymoons_approx}, while the other $18$ dimensions are random Gaussian noise. 


\begin{figure*}[t!]
\centering
\includegraphics[width=0.28\textwidth]{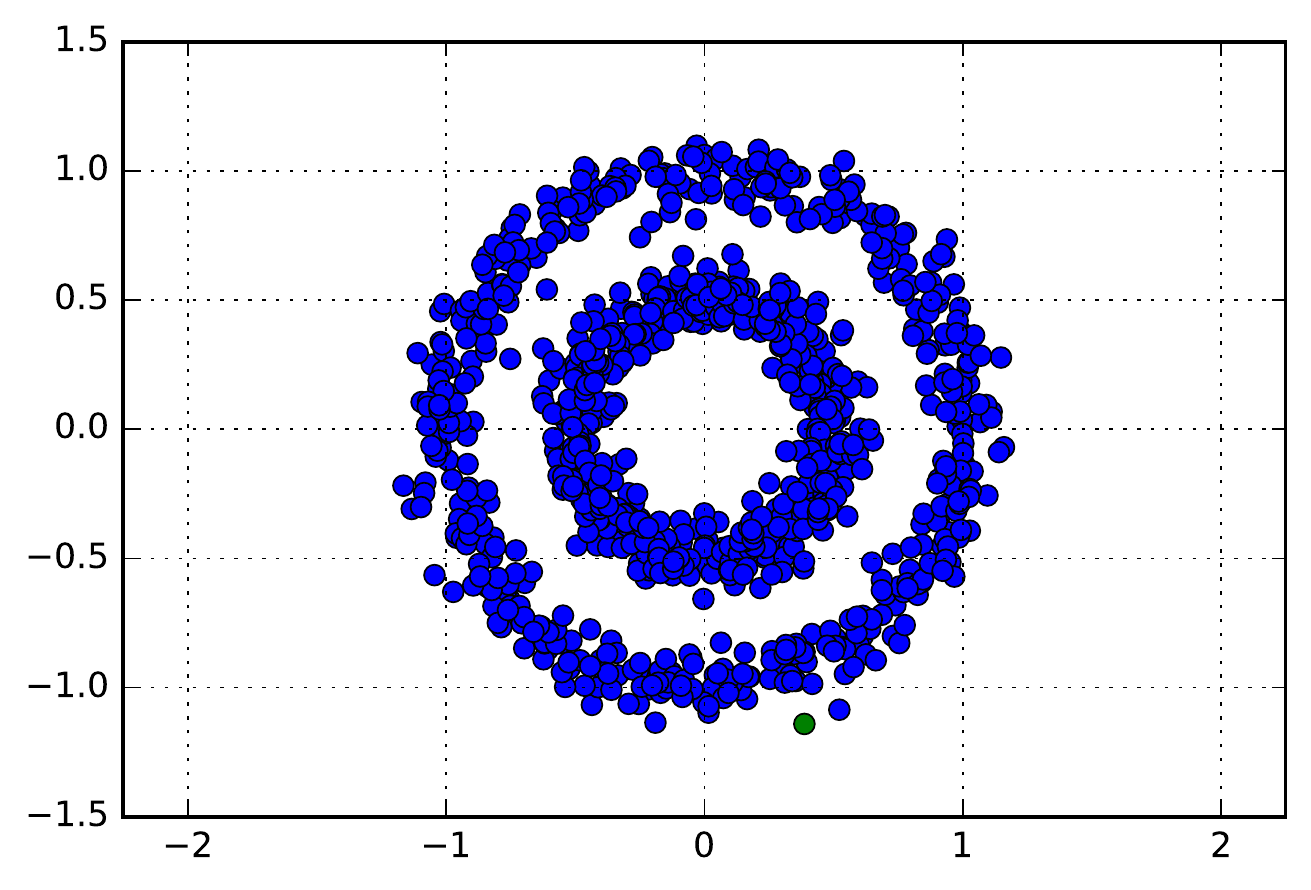}
\includegraphics[width=0.28\textwidth]{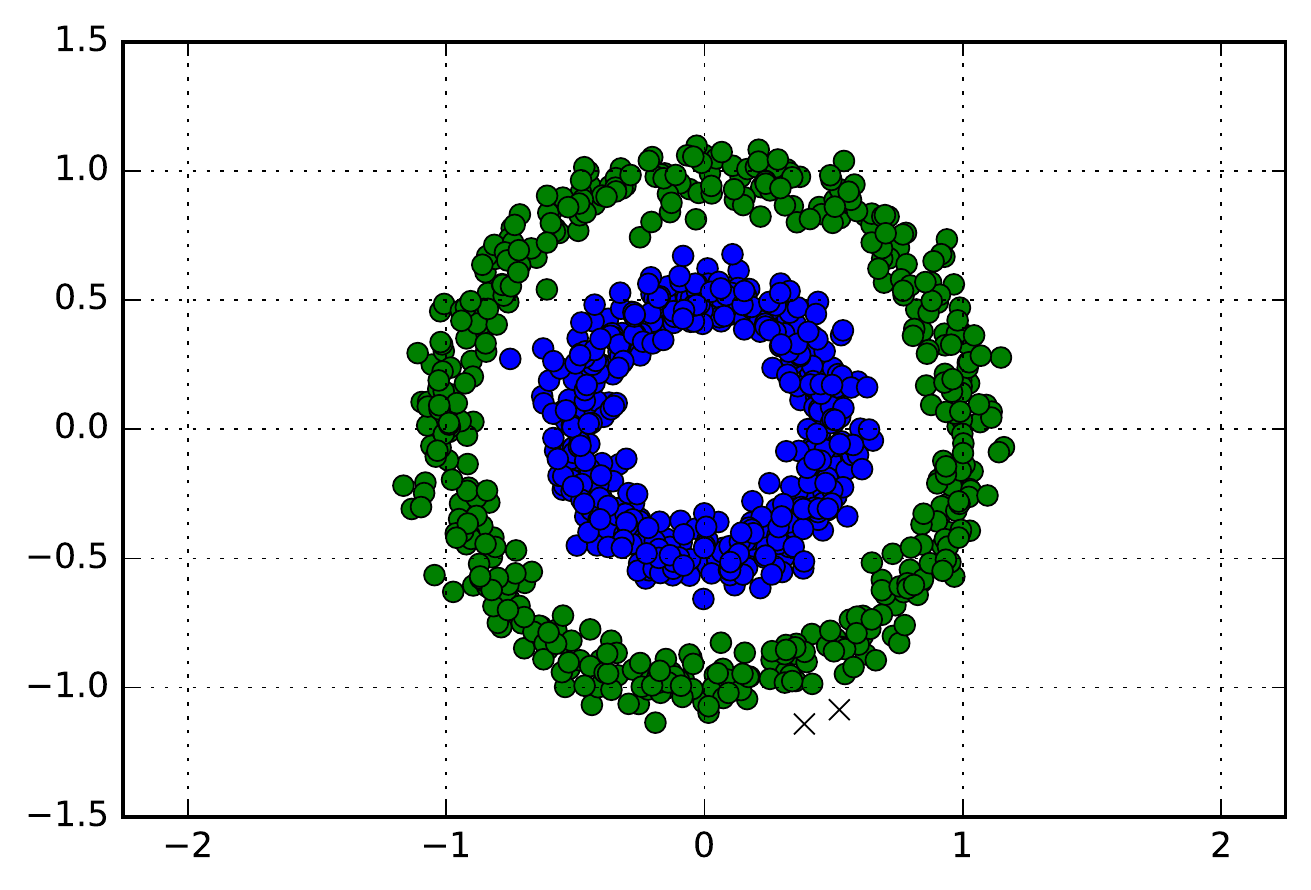}
\includegraphics[width=0.28\textwidth]{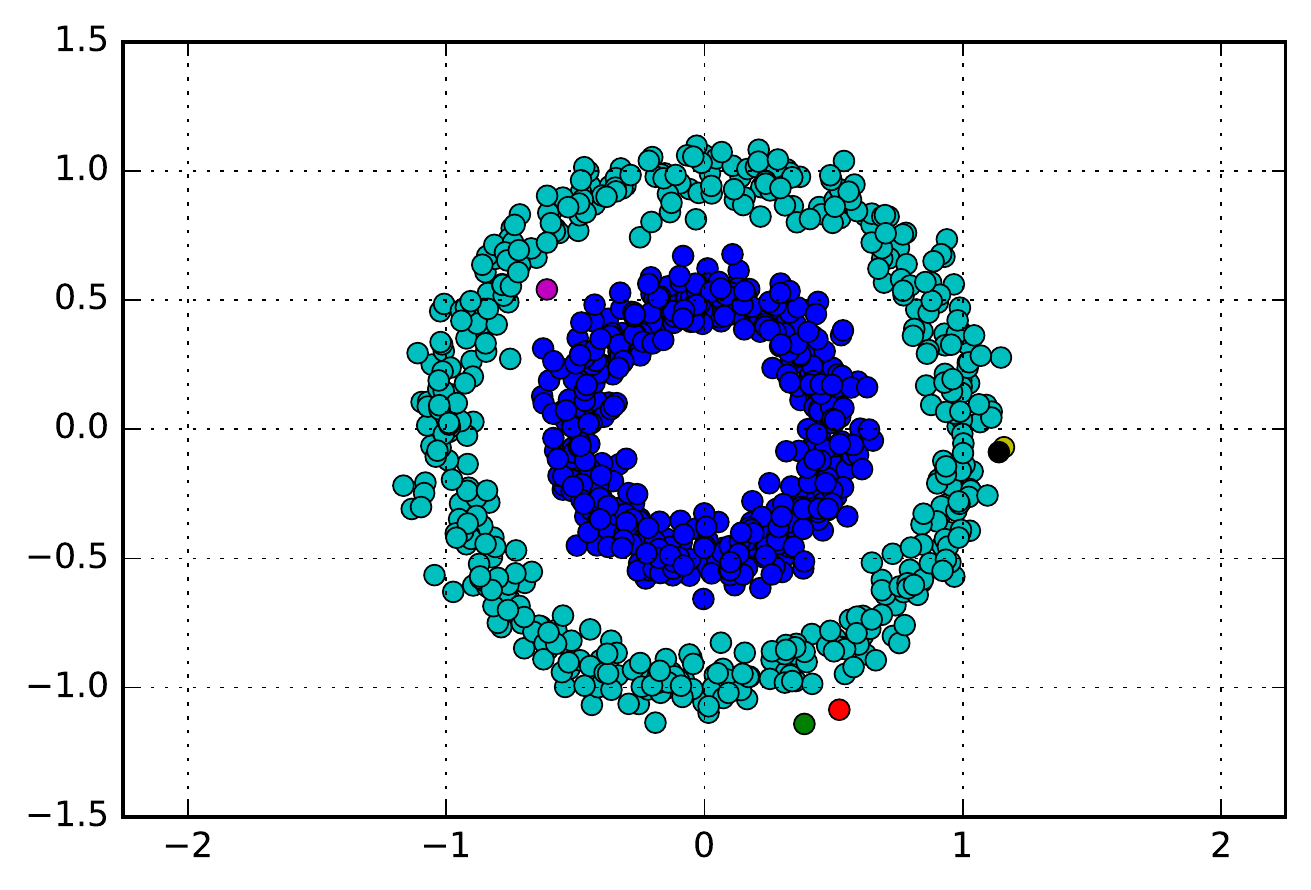}
\vspace*{-0.1in}
\caption{Noisy circles: SEMST, DBSCAN ($\epsilon = 0.15$, $minPts = 5$), DBMSTClu with an approximate MST.}
\label{fig:noisycircles_approx}
\end{figure*}

\begin{figure*}[!t]
\vspace*{-0.1in}
\centering
\includegraphics[width=0.28\textwidth]{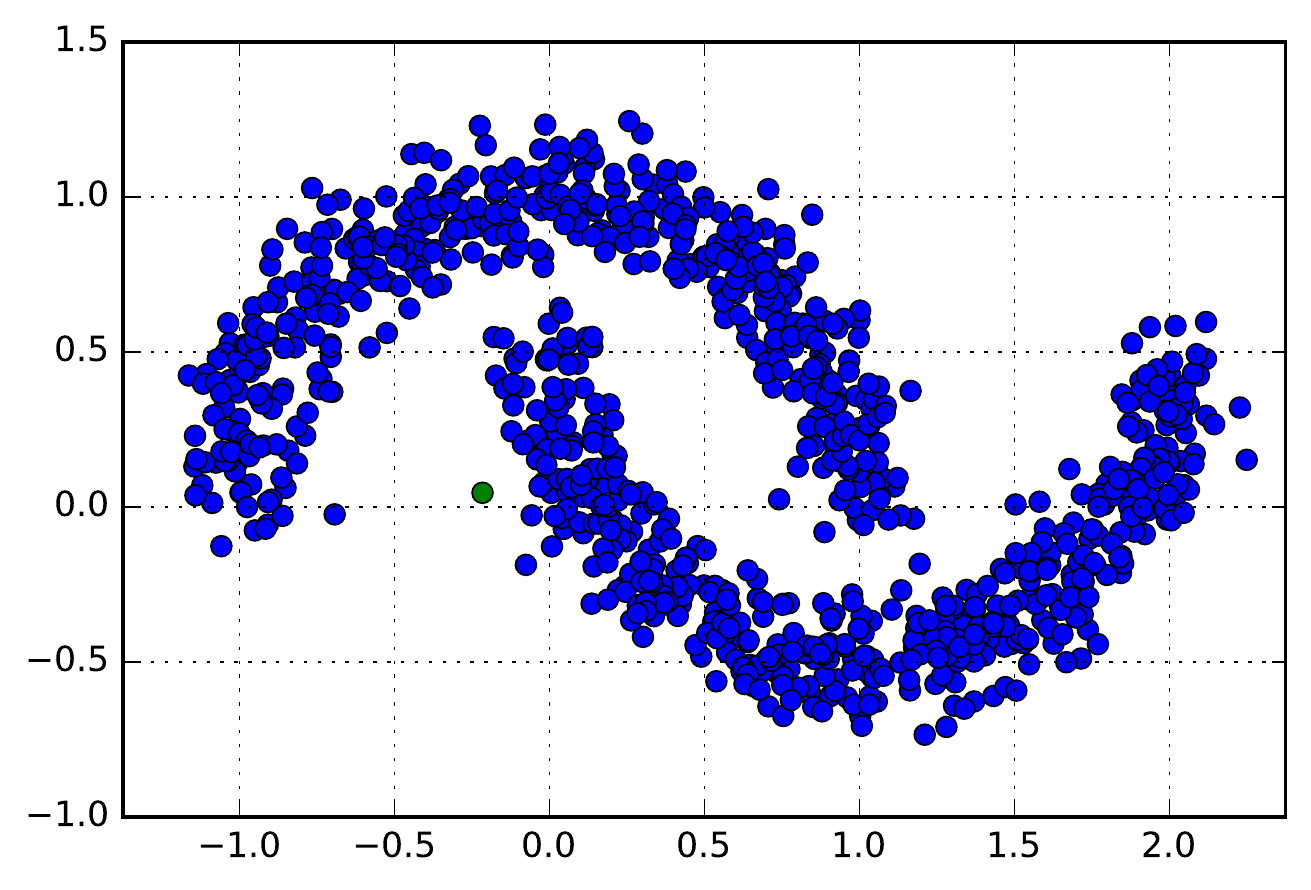}
\includegraphics[width=0.28\textwidth]{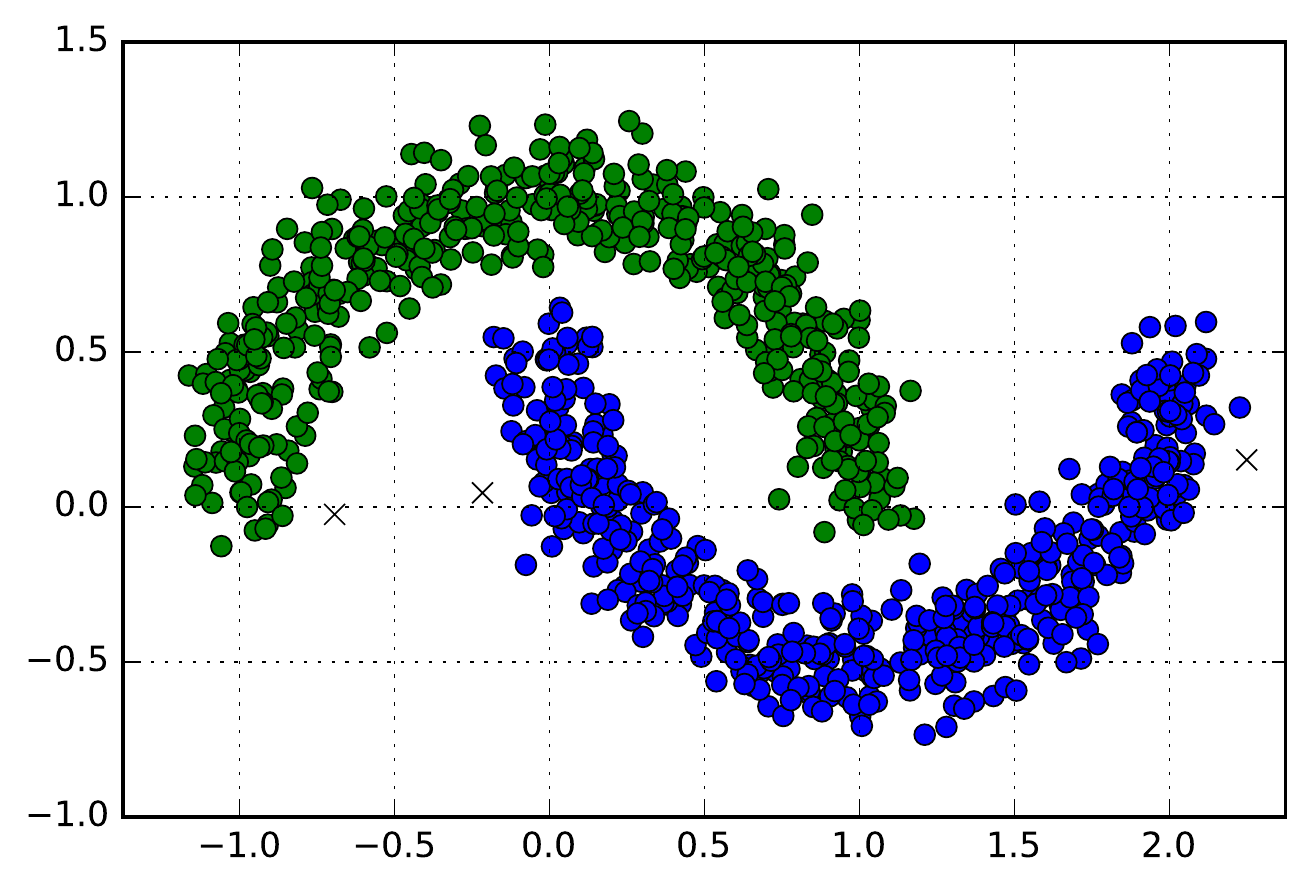}
\includegraphics[width=0.28\textwidth]{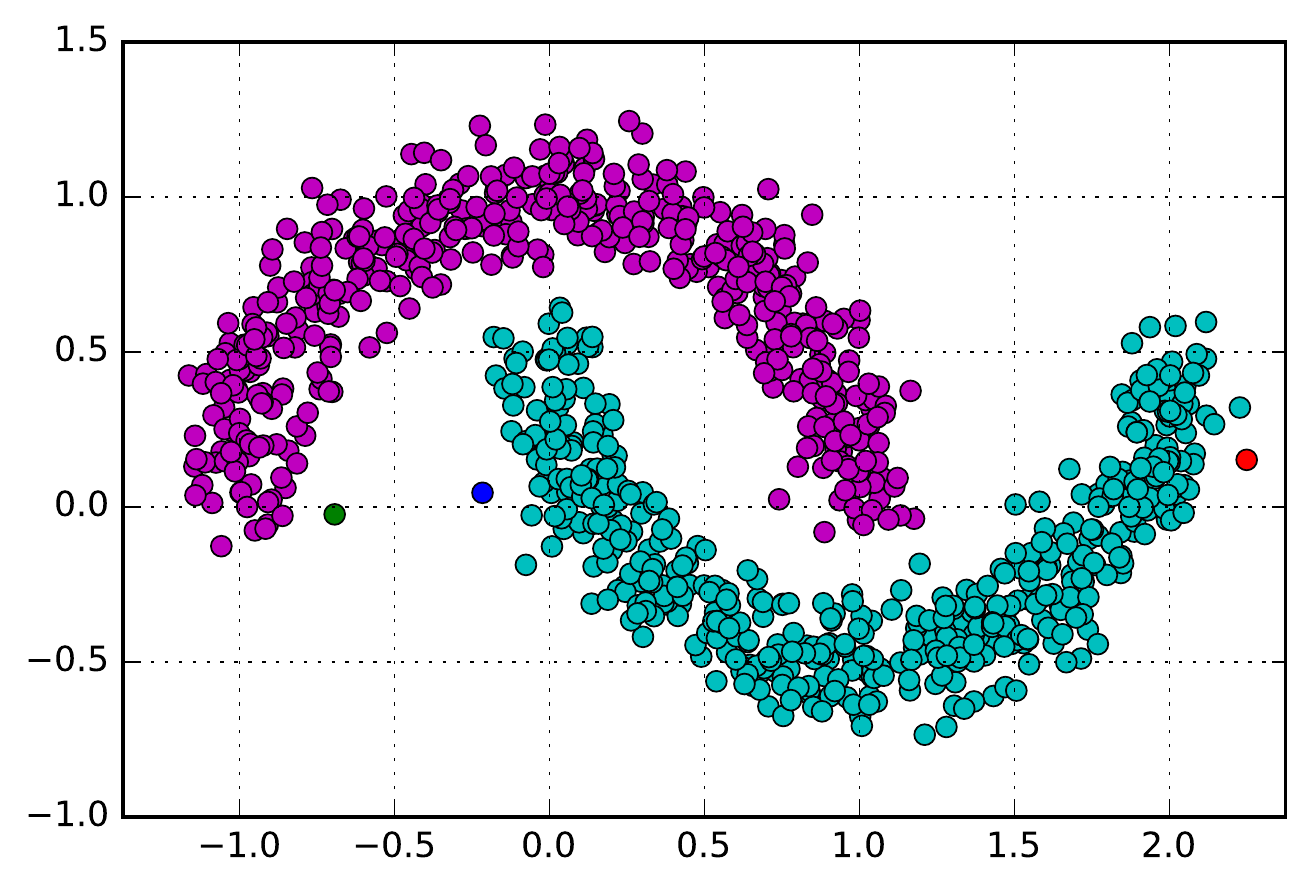}
\vspace*{-0.1in}
\caption{Noisy moons: SEMST, DBSCAN ($\epsilon = 0.16$, $minPts = 5$), DBMSTClu with an approximate MST.}
\label{fig:noisymoons_approx}
\vspace*{-0.1in}
\end{figure*}

\textbf{Real dataset.}
DBMSTClu performances are also measured on the mushroom dataset (\url{https://archive.ics.uci.edu/ml/datasets/mushroom}). 
It contains $8124$ records of $22$ categorical attributes corresponding to $23$ species of gilled mushrooms in the Agaricus and Lepiota family. $118$ binary attributes are created from the $22$ categorical ones, then the complete graph (about $33$ millions of edges) is built by computing the normalized Hamming distance (i.e. the number of distinct bits) between points.

\begin{figure*}[t!]
\centering
\includegraphics[width=0.30\textwidth]{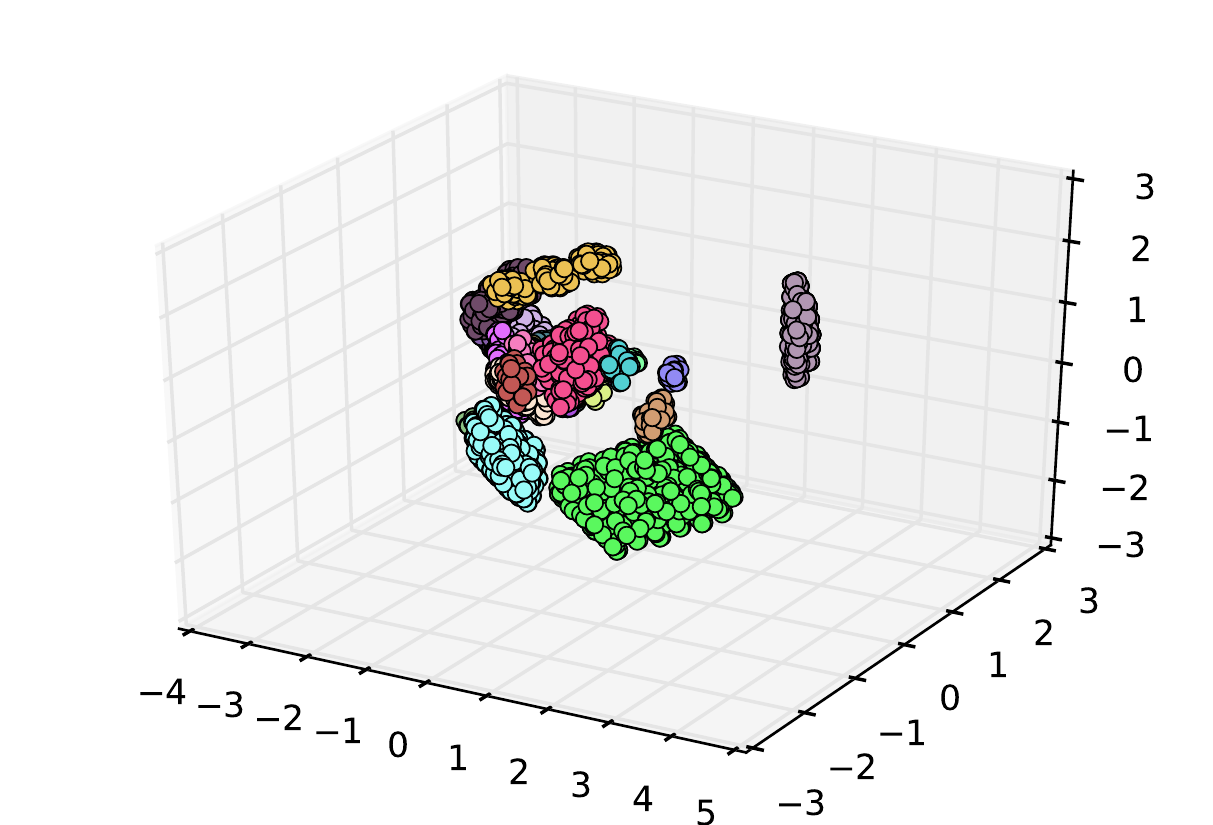}
\includegraphics[width=0.30\textwidth]{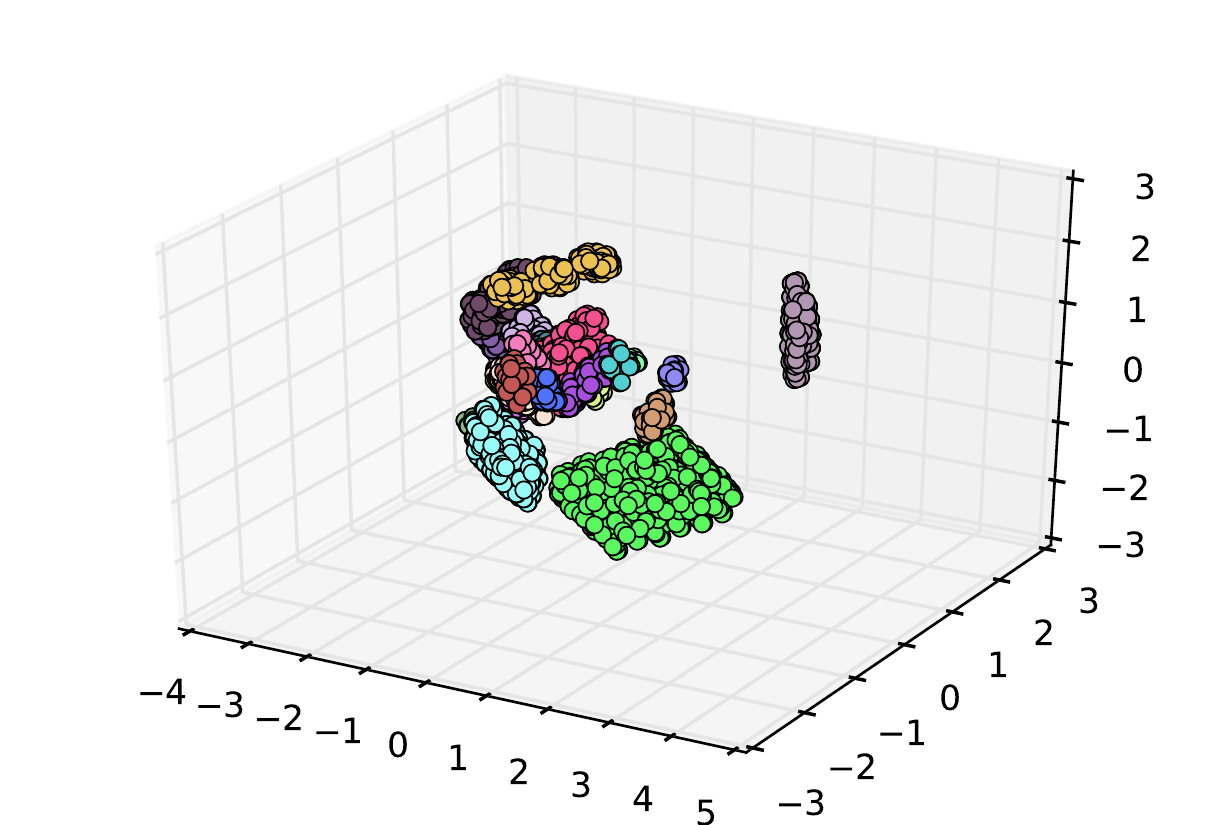}
\includegraphics[width=0.30\textwidth]{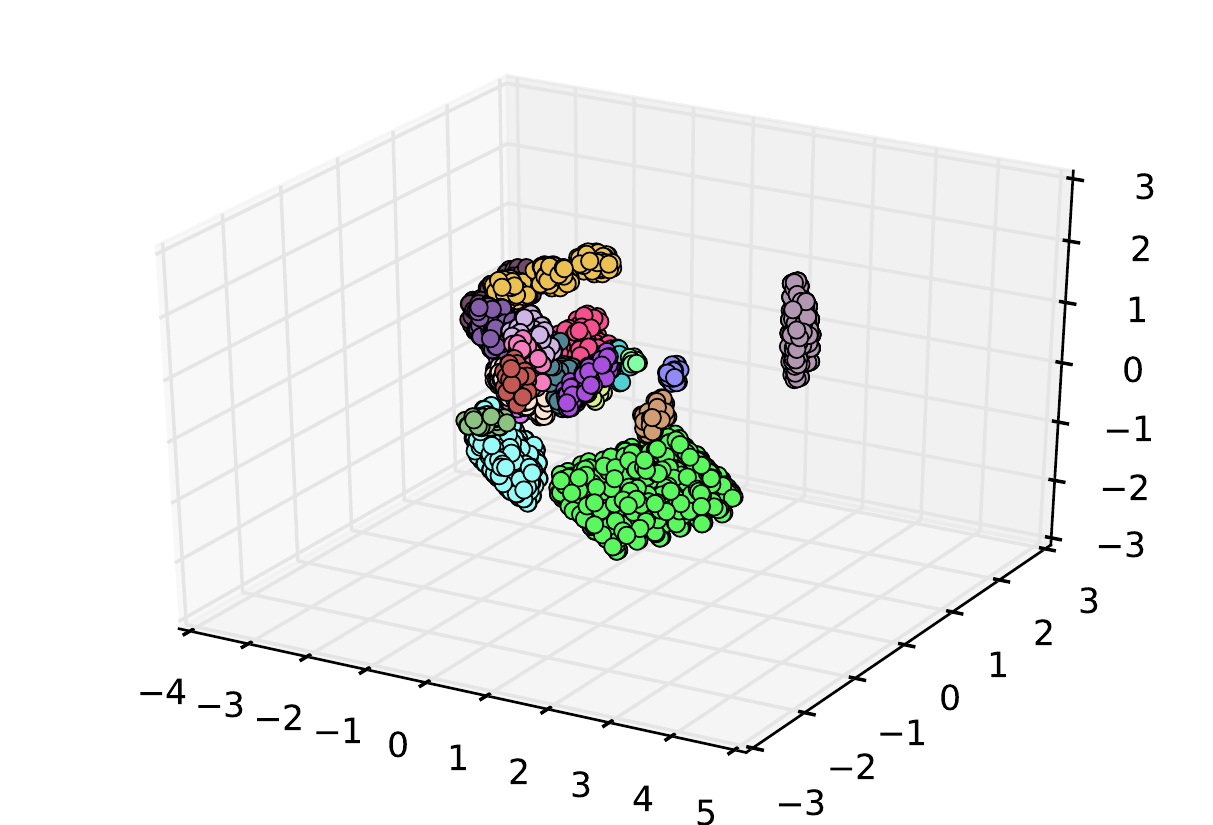}
\caption{Mushroom dataset: SEMST, DBSCAN ($\epsilon = 1.5$, $minPts = 2$), DBMSTClu with an approximate MST (projection on the first three principal components).}
\label{fig:mushrooms_approx}
\vspace*{-0.1in}
\end{figure*}

\textbf{Results.}
Fig.~\ref{fig:noisycircles_approx}, \ref{fig:noisymoons_approx} and~\ref{fig:mushrooms_approx} show the results for all previously defined datasets and aforementioned methods. 
The synthetic datasets were projected onto 2D spaces for visualization purposes. They were produced with a noise level such that SEMST fails and DBSCAN does not perform well without parameters optimization. In particular, for DBSCAN all the cross points correspond to noise. 
With the concentric circles, SEMST does not cut on the consistent edges, hence leads to an isolated singleton cluster. DBSCAN classifies the same point plus a near one as noise while recovering the two circles well. Finally, DBMSTClu finds the two main clusters and also creates five singleton clusters which can be legitimately considered as noise as well.
With noisy moons, while DBSCAN considers three outliers, DBMSTClu detects the same as singleton clusters.
As theoretically proved above, experiments emphasize the fact that our algorithm is more subtle than simply cutting the heaviest edges as the failure of SEMST shows. Moreover our algorithm exhibits an ability to detect outliers, which could be labeled as noise in a postprocessing phase. Another decisive advantage of our algorithm is the absence of any required parameters, contrarily to DBSCAN.
For the mushroom dataset, if suitable parameters are given to SEMST and DBSCAN, the right $23$ clusters get found while DBMSTClu retrieves them without tuning any parameter. 
Quantitative results for the synthetic datasets are shown in Table~\ref{tab:indices}: the achieved silhouette coefficient (between $-1$ and $1$), Adjusted Rand Index (ARI) (between $0$ and $1$) and DBCVI. For all the indices, the higher, the better. Further analysis can be read in supplementary material. For the mushroom dataset, the corresponding DBCVI and silhouette coefficient are resp. $0.75$ and $0.47$.


\vspace{-0.05in}
\begin{table}[!htbp]
\centering
\resizebox{\columnwidth}{!}{%
\begin{tabular}{l|l|l|l|l|l|l}
\cline{2-7}
                               & \multicolumn{2}{l|}{Silhouette coeff.} & \multicolumn{2}{l|}{ARI} & \multicolumn{2}{l|}{DBCVI}                              \\ \hline
\multicolumn{1}{|l|}{SEMST}    & \textbf{0.16}   & -0.12             & 0                   & 0                  & 0.001            & \multicolumn{1}{l|}{0.06}           \\ \hline
\multicolumn{1}{|l|}{DBSCAN}   & 0.02              & \textbf{0.26}  & \textbf{0.99}    & \textbf{0.99}              & -0.26           & \multicolumn{1}{l|}{\textbf{0.15}}          \\ \hline
\multicolumn{1}{|l|}{DBMSTClu} & -0.26             & \textbf{0.26}            & \textbf{0.99}                & \textbf{0.99}   & \textbf{0.18} & \multicolumn{1}{l|}{\textbf{0.15}} \\ \hline
\end{tabular} %
}
\caption{Silhouette coefficients, ARI and DBCVI for the noisy circles (left) and noisy moons (right) datasets. 
\vspace{-0.2in}
} 
\label{tab:indices}
\end{table}

\subsection{Scalability of the clustering.}
\label{sec:scalability}

\begin{table*}[]
\centering
\resizebox{0.7\textwidth}{!}{%
\begin{tabular}{|l|c|c|c|c|c|c|c|c|}
\hline
$K \backslash N$       & \multicolumn{1}{l|}{$1000$} & \multicolumn{1}{l|}{$10000$} & \multicolumn{1}{l|}{$50000$} & \multicolumn{1}{l|}{$100000$} & \multicolumn{1}{l|}{$250000$} & \multicolumn{1}{l|}{$500000$} & \multicolumn{1}{l|}{$750000$} & \multicolumn{1}{l|}{$1000000$} \\ \hline
$5$   & $0.34$                      & $2.96$                       & $14.37$                      & $28.91$                       & $73.04$                       & $148.85$                      & $218.11$                      & $292.25$                       \\ \hline
$20$  & $0.95$                      & $8.73$                       & $43.71$                      & $88.51$                       & $223.18$                      & $449.37$                      & $669.29$                      & $889.88$                       \\ \hline
$100$ & $4.36$                      & $40.25$                      & $201.76$                     & $398.41$                      & $995.42$                      & $2011.79$                     & $3015.61$                     & $4016.13$                      \\ \hline
``$100/5$" & $12.82$ & $13.60$ & $14.04$ & $13.78$ & $13.63$ & $13.52$ & $13.83$ & $13.74$ \\ \hline
\end{tabular} %
}
\caption{Numerical values for DBMSTClu's execution time (in s) varying $N$ and $K$ (avg. on $5$ runs). The last row shows the execution time ratio between $K = 100$ and $K = 5$. 
}
\label{tab:time}
\vspace*{-0.1in}
\end{table*}

For mushroom dataset, DBMSTClu's execution time (avg. on $5$ runs) is $3.36$s while DBSCAN requires $9.00$s. This gives a first overview of its ability to deal with high number of clusters. Further experiments on execution time were conducted on large-scale random weighted graphs generated from the Stochastic Block Model with varying equal-sized number of clusters $K$ and $N$. The scalability of DBMSTClu is shown in Fig.~\ref{fig:execTimeDBMSTClu} and Table~\ref{tab:time} by exhibiting the linear time complexity in $N$. Graphs with $1$M of nodes and $100$ clusters were easily clustered. In Table~\ref{tab:time}, the row with the execution time ratio between $K = 100$ and $K = 5$ illustrates the first trick from \S\ref{sec:implementation} as the observed time ratio is around $2/3$ of the theoretical one $100/5 = 20$.

\begin{figure}[!htbp]
\centering
\includegraphics[width=0.31\textwidth]{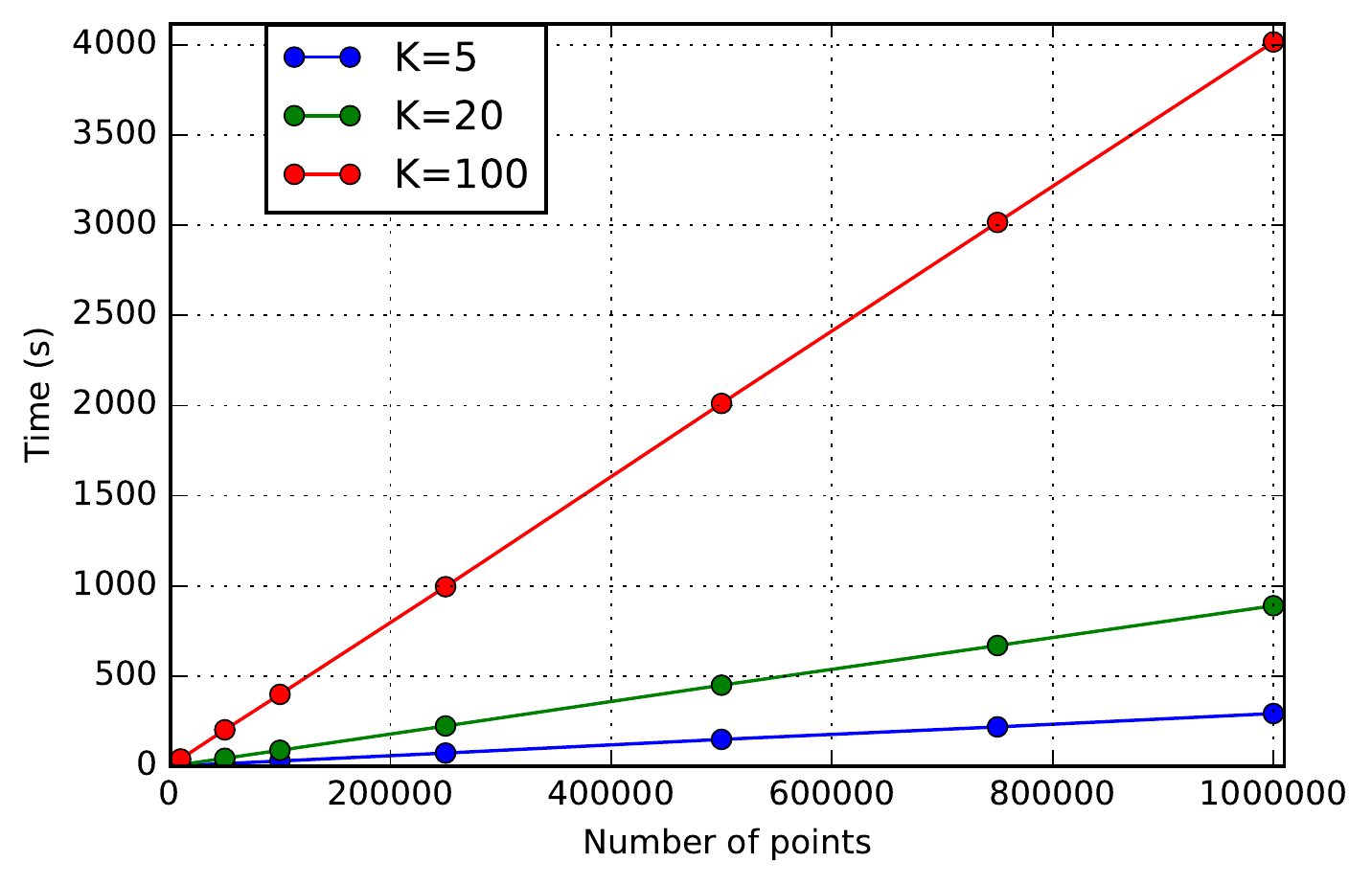}
\vspace*{-0.15in}
\caption{DBMSTClu's execution time with values of $N \in \{ 1\mbox{K}, 10\mbox{K}, 50\mbox{K}, 100\mbox{K}, 250\mbox{K}, 500\mbox{K}, 750\mbox{K}, 1\mbox{M} \}$.}
\label{fig:execTimeDBMSTClu}
\end{figure}

\vspace*{-0.2in}
\section{Conclusion} \label{sec:ccl}
In this paper we introduced DBMSTClu a novel \emph{space-efficient} Density-Based Clustering algorithm which only relies on a Minimum Spanning Tree (MST) of the dissimilarity graph $\mathcal{G}$: the spatial and time costs are resp. $O(N)$ and $O(NK)$ with $N$ the number of data points and $K$ the number of clusters. This enables to deal easily with graphs of million of nodes. 
Moreover, DBMSTClu is \emph{non-parametric}, unlike most existing clustering methods: it automatically determines the right number of nonconvex clusters. Although the approach is fundamentally independent from the sketching phase, its robustness has been assessed by using as input an approximate MST of the sketched $\mathcal{G}$ rather than an exact one. The graph sketch is computed dynamically on the fly as new edge weight updates are read in only one pass over the data. This brings a space-efficient solution for finding an MST of $\mathcal{G}$ when the $O(N^2)$ edges cannot fit in memory. Hence, our algorithm adapts to the semi-streaming setting with $O(N \polylog N)$ space. Our approach shows promising results, as evidenced by the experimental part regarding time and space scalability and clustering performance even with sketching.
Further work would consist in using this algorithm in privacy issues, as the lost information when sketching might ensure data privacy. Moreover, as it is already the case for the graph sketching, we could look for adapting both the MST recovery and DBMSTClu to the fully online setting, i.e. to be able to modify dynamically current MST and clustering partition as a new edge weight update from the stream is seen. 

\newpage

\input{appendix.tex}

\end{document}

%% file: appendix.tex
\fancyfoot[R]{\footnotesize{\textbf{Copyright \textcopyright\ 2018 by SIAM\\
Unauthorized reproduction of this article is prohibited}}}





\section{Proofs.}

\subsection{Proof of Prop. 4.1.}

\textsc{\vspace*{+0.1in} \\ Proposition 4.1. When the first cut is not the heaviest} \textit{Let $\mathcal{T}$ be an MST of the dissimilarity data graph with $N$ nodes. Let us consider this specific case: all edges have a weight equal to $w$ except two edges $e_1$ and $e_2$ resp. with weight $w_1$ and $w_2$ s.t. $w_1 > w_2 > w > 0$. 
DBMSTClu does not cut any edge with weight $w$ and cuts $e_2$ instead of $e_1$ as a first cut iff:
\begin{equation*}
w_2 > \frac{2n_2w_1 - n_1 + \sqrt{ n_1^2 + 4 w_1( n_2^2 w_1 + N^2 - N n_1 -n_2 ^2)} }{2(N - n_1 + n_2)}
\end{equation*}
where $n_1$ (resp. $n_2$) is the number of nodes in the first cluster resulting from the cut of $e_1$ (resp. $e_2$). Otherwise, $e_1$ gets cut.}

\begin{proof} 
Let $DBCVI_1$ (resp. $DBCVI_2$) be the DBCVI after cut of $e_1$ (resp. $e_2$). 
As $w$ (resp. $w_1$) is the minimum (resp. maximal) weight, the algorithm does not cut $e$ since the resulting DBCVI would be negative (cf. Lemma~4.2) while $DBCVI_1$ is guaranteed to be positive (cf. Lemma~4.1). 
So, the choice will be between $e_1$ and $e_2$ but $e_2$ gets cut iff $DBCVI_2  > DBCVI_1$. 
$DBCVI_1$ and $DBCVI_2$ expressions are simplified w.l.o.g. by scaling the weights by $w$ s.t. $w \leftarrow 1$, $w_1 \leftarrow w_1 / w$, $w_2 \leftarrow w_2 / w$, hence $w_1 > w_2 > 1$. Then,
\begin{align*}
& DBCVI_2 > DBCVI_1 > 0 \\
&\iff \frac{n_2}{N} ( \frac{w_2}{ w_1} - 1 ) + ( 1 - \frac{n_2}{N} ) ( 1 - \frac{1}{w_2} ) \\
& \qquad \qquad - \frac{n_1}{N} ( 1 - \frac{1}{w_1} ) + ( 1 - \frac{n_1}{N} ) ( 1 - \frac{w_2}{w_1} ) > 0   \\
&\iff w_2^2 \underbrace{(N + n_2 - n_1)}_{a} + w_2 \underbrace{(n_1 - 2n_2w_1)}_{b} \\
& \qquad \qquad + \underbrace{(n_2-N)w_1}_{c < 0} > 0.
\end{align*}
Clearly, $\Delta = b^2 - 4ac$ is positive and $c/a$ is negative. But $w_2 > 0$, then $w_2 > \frac{-b + \sqrt{b^2 - 4ac}}{2a}$ which gives the final result after some simplifications.
\end{proof}

\subsection{Proof of Prop. 4.2.}
\textsc{\vspace*{+0.1in} \\ Proposition 4.2. First cut on the heaviest edge in the middle} \textit{Let $\mathcal{T}$ be an MST of the dissimilarity data graph with $N$ nodes. Let us consider this specific case: all edges have a weight equal to $w$ except two edges $e_1$ and $e_2$ resp. with weight $w_1$ and $w_2$ s.t. $w_1 > w_2 > w > 0$. Denote $n_1$ (resp. $n_2$) the number of nodes in the first cluster resulting from the cut of $e_1$ (resp. $e_2$). 
In the particular case where edge $e_1$ with maximal weight $w_1$ stands between two subtrees with the same number of points, i.e. $n_1 = N/2$, $e_1$ is always preferred over $e_2$ as the first optimal cut.}

\begin{proof}
A reductio ad absurdum is made by showing that cutting edge $e_2$ i.e. $DBCVI_2  > DBCVI_1$ leads to the contradiction $w_1/w < 1$. With the scaling process from Prop.~4.1'proof:
\begin{align*}
DBCVI_1 & = \frac{1}{2} ( 1 - \frac{1}{w_1} ) + \frac{1}{2} ( 1 - \frac{w_2}{w_1} )  = 1 - \frac{1}{2 w_1} - \frac{w_2}{2 w_1} \nonumber \\
DBCVI_2 & = \frac{n_2}{N} ( \frac{w_2}{ w_1} - 1 ) + ( 1 - \frac{n_2}{N} ) ( 1 - \frac{1}{w_2} )  \nonumber \\
& = 1 - \frac{1}{w_2} + \frac{n_2}{N} (\underbrace{\frac{w_2}{w_1} + \frac{1}{w_2} - 2}_{= A}) \nonumber
\end{align*}
There is $w_2 > w = 1$, so $\frac{1}{w_2} < 1$. Besides $w_2 < w_1$ so $\frac{w_2}{w_1} < 1$ thus, $A < 0$. Let now consider w.l.o.g. that edge $e_2$ is on the "right side" (right cluster/subtree) of $e_1$ (similar proof if $e_2$ is on the left side of $e_1$). Hence, it is clear that for maximizing $DBCVI_2$ as a function of $n_2$, we need $n_2 = n_1 + 1$. Then,
\begin{align*}
& DBCVI_2 > DBCVI_1 \\
&\iff -\frac{1}{w_2} + ( \frac{1}{2} + \frac{1}{N})( \frac{w_2}{w_1} - 2 + \frac{1}{w_2}) > -\frac{1}{w_1} - \frac{w_2}{w_1} \\
& \iff (\frac{1}{2w_1} + \frac{1}{N w_1} + \frac{1}{2 w_1}) w_2 - 1 - \frac{2}{N} + \frac{1}{2 w_1} \\
& \qquad \qquad + (-1 + \frac{1}{2} + \frac{1}{N}) \frac{1}{w_2} > 0 \\
& \iff \underbrace{(1 + \frac{1}{N})}_{a > 0}w_2^2 + w_2 \underbrace{( \frac{1}{2} - w_1( 1 + \frac{2}{N}))}_{b < 0} \\
& \qquad \qquad + w_1 \underbrace{(\frac{1}{N} - \frac{1}{2})}_{c < 0} > 0  
\end{align*}
As $c / a < 0$ and $w_2 > 0$, 
$w_2  > \frac{N}{2 (N + 1)} \ [ \ w_1 (1 + \frac{2}{N}) - \frac{1}{2} + \sqrt{ \Delta } \ ]$
with $\Delta = (w_1 ( 1 + \frac{2}{N}) - \frac{1}{2})^2 + 4(1 + \frac{1}{N})(\frac{1}{2} - \frac{1}{N})w_1$.
This inequality is incompatible with $w_1 > w_2$ since:
\begin{align*}
& w_1 > w_2 \iff w_1 > \frac{N}{2 (N + 1)} \ [ \ w_1 (1 + \frac{2}{N}) - \frac{1}{2} + \sqrt{ \Delta } \ ] \\
& \iff w_1 + \frac{1}{2} > \sqrt{\Delta} \\
& \iff \frac{4}{N} \ w_1^2 \ (1 + \frac{1}{N}) + \frac{4}{N} w_1 (-1 - \frac{1}{N}) < 0 \\
&\iff w_1 < 1 \ : ILLICIT 
\end{align*}
Indeed, after the scaling process, $w_1 < 1 = w$ is not possible since by hypothesis, $w_1 > w$.
Finally, it is not allowed to cut $e_2$, the only remaining possible edge to cut is $e_1$.
\end{proof}

\subsection{Proof of Prop. 4.3.}
\textsc{\vspace*{+0.1in} \\ Proposition 4.3. Fate of negative $V_C$ cluster} \textit{Let $K = t + 1$ be the number of clusters in the clustering partition at iteration $t$. If for some $i \in [K], \ V_C(C_i) < 0$, then DBMSTClu will cut an edge at this stage.}

\begin{proof}
Let $i \in [K]$ s.t. $V_C(C_i) < 0$ i.e. $\SEP(C_i) < \DISP(C_i)$. We denote $w_{sep}^l$ the minimal weight outing cluster $C_i$ and $w_{max}$ the maximal weight in subtree $S_i$ of $C_i$ i.e. $\SEP(C_i) \overset{def}{=} w_{sep}^l$ and $\DISP(C_i) \overset{def}{=} w_{max}$. Hence, $w_{sep}^l < w_{max}$. By cutting the cluster $C_i$ on the edge with weight $w_{max}$, we define $C_i^l$ and $C_i^r$ resp. the left and right resulting clusters.

Let us look at $V_C(C_i^l)$. If $\SEP(C_i^l) \geq \DISP(C_i^l)$ then $V_C(C_i^l) \geq 0 \geq V_C(C_i)$ else $V_C(C_i^l) = \frac{\SEP(C_i^l)}{\DISP(C_i^l)}-1$. The definition of the Separation as a minimum and our cut imply that 
$$\SEP(C_i^l) \geq \min(\SEP(C_i), w_{max}) \geq \SEP(C_i).$$
Also the definition of the Dispersion as a maximum implies that $\DISP(C_i^l) \leq \DISP(C_i)$. Hence we get that $\frac{\SEP(C_i^l)}{\DISP(C_i^l)}-1 \geq \frac{\SEP(C_i)}{\DISP(C_i)}-1$ i.e. $V_C(C_i^l) \geq V_C(C_i)$ in this case too.
The same reasoning holds for $C_i^r$ showing that $V_C(C_i^r) \geq V_C(C_i)$.
Finally, 
\begin{align*}
DBCVI_{after cut} &= \sum_{j \neq i} \frac{n_j}{N} V_C(C_j) + \frac{n_i^l}{N}V_C(C_i^l) + \frac{n_i^r}{N}V_C(C_i^r) \\
& \geq \sum_{j \neq i} \frac{n_j}{N} V_C(C_j) + \frac{n_i^l}{N}V_C(C_i) + \frac{n_i^r}{N}V_C(C_i) \\
& = DBCVI_{before cut}.
\end{align*}
Hence cutting the edge with maximal weight in $C_i$ improves the resulting DBCVI.

\begin{figure}
\centering
\includegraphics[width=0.5\textwidth]{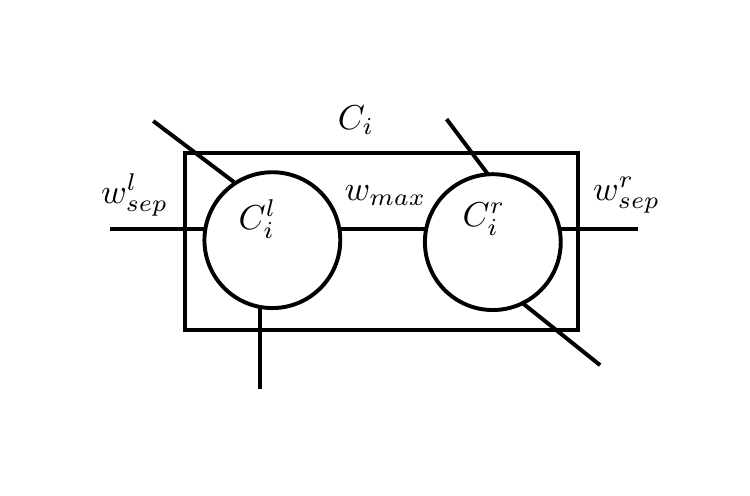}
\vspace*{-0.6in}
\caption{Generic example for proof of Prop.~4.3 and 4.5.}
\label{fig:negativeClusters}
\end{figure}
\end{proof}

\subsection{Proof of Prop. 4.4.}
\textsc{\vspace*{+0.1in} \\ Proposition 4.4. Fate of positive $V_C$ cluster I} \textit{Let $\mathcal{T}$ be an MST of the dissimilarity data graph and $C$ a cluster s.t. $V_C(C) > 0$ and $\SEP(C) = s$. DBMSTClu does not cut an edge $e$ of $C$ with weight $w < s$ if both resulting clusters have at least one edge with weight greater than $w$.}

\begin{proof} 
Let us consider clusters $C_1$ and $C_2$ resulting from the cut of edge $e$.
Assume that in the associated subtree of $C_1$ (resp. $C_2$), there is an edge $e_1$ (resp. $e_2$) with a weight $w_1$ (resp. $w_2$) higher than $w$ s.t. without loss of generality, $w_1 > w_2$. 
Since $V_C(C) > 0$, $s > w_1 > w_2 > w$.
But cutting edge $e$ implies that for $i \in \{1, 2 \}$, $\DISP(C_i) > \SEP(C_i) = w$, and thus $V_C(C_i) < 0$. Cutting edge $e$ would therefore mean to replace a cluster $C$ s.t. $V_C(C) > 0$ by two clusters s.t. for $i \in \{1, 2 \}$, $V_C(C_i) < 0$ which obviously decreases the current DBCVI. Thus, $e$ does not get cut at this step of the algorithm.
\end{proof}

\subsection{Proof of Prop. 4.5.}

\textsc{\vspace*{+0.1in} \\ Proposition 4.5. Fate of positive $V_C$ cluster II} \textit{Consider a partition with $K$ clusters s.t. some cluster $C_i$, $i \in [K]$ with $V_C(C_i) > 0$ is in the setting of Fig.~\ref{fig:negativeClusters} i.e. cutting the heaviest edge $e$ with weight $w_{max}$ results in two clusters: the left (resp. right) cluster $C_i^l$ (resp. $C_i^r$) with $n_1$ points (resp. $n_2$) s.t. $\DISP(C_i^l) = d_1$, $\SEP(C_i^l) = w^l_{sep}$, $\DISP(C_i^r) = d_2$ and $\SEP(C_i^r) = w^r_{sep}$. Assuming w.l.o.g. $w^l_{sep} > w^r_{sep}$, 
cutting edge $e$ improves the DBCVI iff:
\begin{equation*}
\frac{ \left(  \frac{n_1 d_1 + n_2 d_2}{n_1 + n_2} \right) }{w_{max}} \leq \frac{w_{max}}{w^r_{sep}}.
\end{equation*}}

\begin{proof}
As $V_C(C_i) > 0$, there is $\SEP(C_i) = w^r_{sep} > w_{max}$. Then, the DBCVI before ($K$ clusters) and after cut of $w_{max}$ ($K+1$ clusters) are:
\begin{align*}
DBCVI_{K} &= \sum^K_{j \neq i} V_C(C_j) + \frac{n_1 + n_2}{N} \left(1 - \frac{w_{max}}{w^r_{sep}} \right) \\
DBCVI_{K+1} &= \sum^K_{j \neq i} V_C(C_j) + \frac{n_1}{N} \left(1 - \frac{d_1}{w_{max}} \right)  \\
& \qquad \qquad + \frac{n_2}{N} \left(1 - \frac{d_2}{w_{max}} \right)
\end{align*}
DBMSTClu cuts $w_{max}$ iff $DBCVI_{K+1} \geq DBCVI_{K}$. So the result after simplification.
\end{proof}

\section{Complements on experiments.}

Experiments were conducted using Python and scikit-learn library~\cite{scikit-learn} on a single-thread process on an intel processor based node.

\subsection{Safety of the sketching.}

\begin{figure*}[h]
\centering
\includegraphics[width=0.32\textwidth]{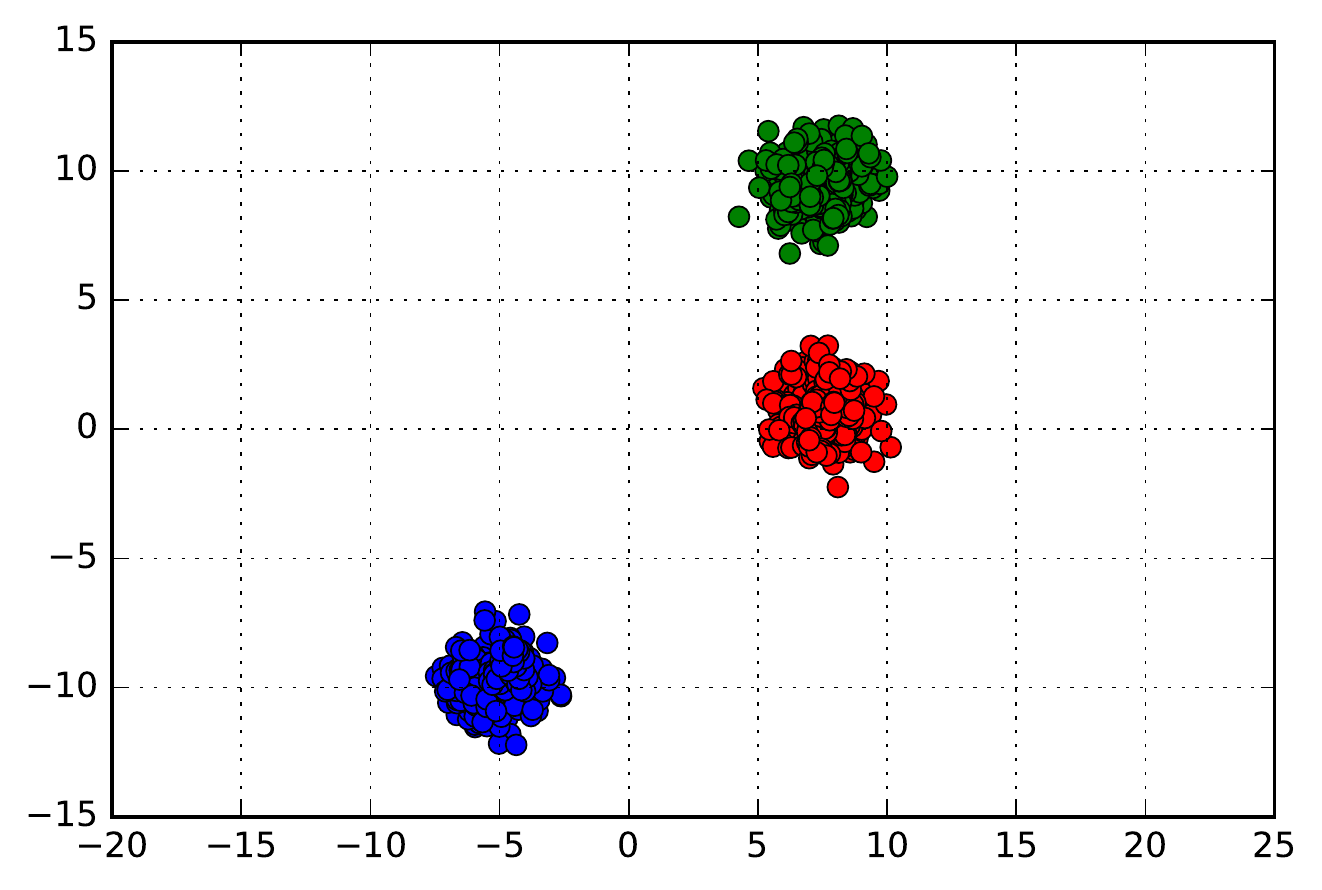}
\includegraphics[width=0.32\textwidth]{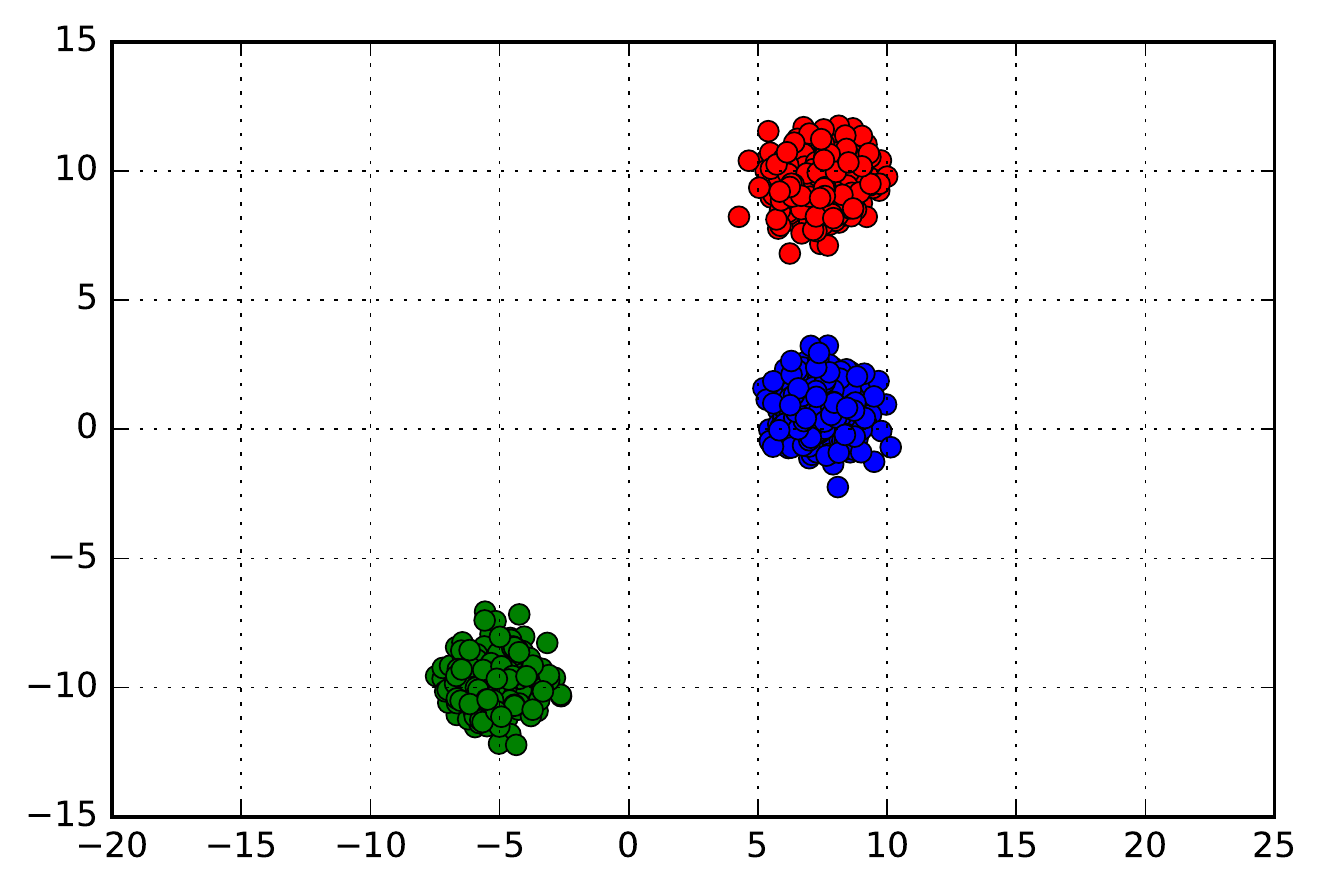}
\includegraphics[width=0.32\textwidth]{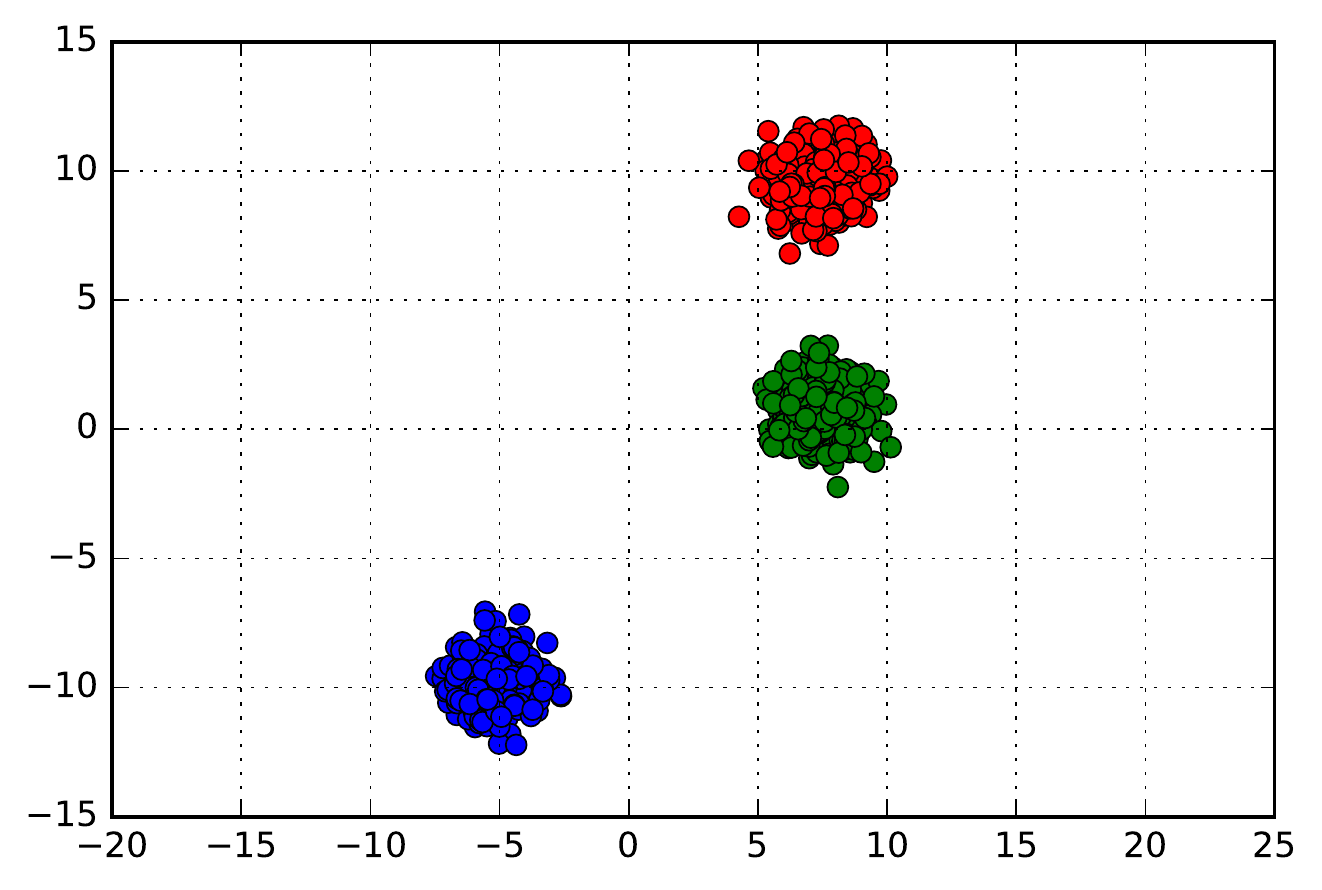}
\caption{Three blobs: SEMST, DBSCAN ($\epsilon = 1.4$, $minPts = 5$), DBMSTClu with an approximate MST.}
\label{fig:3blobs_approx}
\end{figure*}

Fig.~\ref{fig:3blobs_approx} shows another result on a synthetic dataset: three blobs generated from three Gaussian distributions.
With the three blobs, each method SEMST, DBSCAN and DBMSTClu performs well: they all manage to retrieve three clusters.

Quantitative results for the three synthetic datasets are shown in Table~\ref{tab:indices}: the achieved silhouette coefficient, Adjusted Rand Index (ARI) and DBCVI. For all the indices, the higher, the better. Silhouette coefficient (between $-1$ and $1$) is used to measure a clustering partition without any external information. For DBSCAN it is computed by considering noise points as singletons. We see that this measure is not very suitable for nonconvex clusters like noisy circles or moons. The ARI (between $0$ and $1$) measures the similarity between the experimental clustering partition and the known groundtruth. DBSCAN and DBMSTClu give similar almost optimal results. Finally, the obtained DBCVIs are consistent, since the best ones are reached for DBMSTClu. 

\begin{table*}[h]
\centering
\begin{tabular}{l|l|l|l|l|l|l|l|l|l|}
\cline{2-10}
                               & \multicolumn{3}{l|}{Silhouette coeff.}           & \multicolumn{3}{l|}{Adjusted Rand Index}     & \multicolumn{3}{l|}{DBCVI}                      \\ \hline
\multicolumn{1}{|l|}{SEMST}    & \textbf{0.84} & \textbf{0.16} & -0.12          & \textbf{1} & 0              & 0              & \textbf{0.84} & 0.001          & 0.06         \\ \hline
\multicolumn{1}{|l|}{DBSCAN}   & \textbf{0.84}          & 0.02          & \textbf{0.26} & \textbf{1}      & \textbf{0.99} & \textbf{0.99}          & \textbf{0.84}          & -0.26         & \textbf{0.15}       \\ \hline
\multicolumn{1}{|l|}{DBMSTClu} & \textbf{0.84} & -0.26         & \textbf{0.26}          & \textbf{1} & \textbf{0.99}           & \textbf{0.99} & \textbf{0.84} & \textbf{0.18} & \textbf{0.15} \\ \hline
\end{tabular}
\caption{Silhouette coefficients, Adjusted Rand Index and DBCVI for the blobs, noisy circles and noisy moons datasets with SEMST, DBSCAN and DBMSTClu. 
} 
\label{tab:indices}
\end{table*}